\documentclass{article}
\usepackage{fullpage}
\usepackage{amsmath}
\usepackage{amsthm}
\usepackage{amssymb}
\usepackage{graphicx}
\usepackage[ruled]{algorithm2e}
\usepackage[utf8]{inputenc}
\usepackage{times}
\usepackage{color}
\newtheorem{theorem}{Theorem}
\newtheorem{lemma}{Lemma}
\newtheorem{remark}{Remark}
\usepackage{algpseudocode}
\newcommand{\Ex}{\mathbb{E}}
\newcommand{\ExR}[1]{\mathbb{E}\left[ #1 \right]}
\newcommand{\ExB}[1]{\mathbb{E}\left[ #1 \right]}
\newcommand{\condR}{\ \vline\  }
\newcommand{\condB}{\ \vline\ }
\newcommand{\comment}[1]{}
\newcommand{\Beta}{\text{Beta}}
\newtheorem{fact}{Fact}
\newcommand{\remove}[1]{\textcolor{red}{}}

\newcommand{\add}[1]{#1}

\title{Analysis of Thompson Sampling for the multi-armed bandit problem}

 \author{Shipra Agrawal \\ 
 Microsoft Research India\\
 \texttt{shipra@microsoft.com} 
 \and
 Navin Goyal\\
 Microsoft Research India\\
 \texttt{navingo@microsoft.com}\\
 }
 
\begin{document}
\date{}
\maketitle

\begin{abstract}
The multi-armed bandit problem is a popular model for studying exploration/exploitation trade-off in sequential decision
problems. Many algorithms are now available for this well-studied problem. 
One of the earliest algorithms, given by W. R. Thompson, dates back to 1933. This algorithm, referred to as Thompson Sampling, is a natural  Bayesian algorithm. The basic idea is to choose an arm to play according to its 
probability of being the best arm. Thompson Sampling algorithm has experimentally been shown to be close to optimal. In addition, it is efficient to implement and exhibits several desirable properties such as small regret for delayed feedback.
However, theoretical understanding of this algorithm was quite limited. In this paper, for the first time,
we show that Thompson Sampling algorithm achieves logarithmic expected regret for the stochastic multi-armed bandit problem. More precisely, for the stochastic two-armed  bandit problem, the expected regret in time $T$ is $O(\frac{\ln T}{\Delta} + \frac{1}{\Delta^3})$. And, for the stochastic $N$-armed  bandit problem, the expected regret in time $T$ is $O(\left[(\sum_{i=2}^N \frac{1}{\Delta_i^2})^2\right] \ln T)$. Our bounds are optimal  but for the dependence on $\Delta_i$ and the constant factors in big-Oh. 
\end{abstract}

%\begin{keywords}
%multi-armed bandit, Thompson Sampling, Bayesian algorithm, online learning
%\end{keywords}

\section{Introduction}
Multi-armed bandit (MAB) problem models the exploration/exploitation trade-off inherent in sequential decision problems. 
%One of the early motivations for studying MAB problem was clinical trials: suppose that we have $N$ different treatments of unknown efficacy for a certain disease. Patients arrive sequentially, and we must decide on a treatment to administer for each arriving patient. To make this decision, we could learn from how the previous choices of treatments fared for the previous patients. After a sufficient number of trials, we may have a reasonable idea of which treatment is most effective, and from then on, we could administer that treatment for all the patients. However, initially, when there is no or very little information available, we need to \emph{explore} and try each treatment sufficient number of times. We wish to do this exploration in such a way that we can find the best treatment and start \emph{exploiting} it as soon as possible. The MAB problem is to decide how to choose the treatment for the next patient, given the outcomes of the treatments so far. Today, multi-armed bandit problem has a diverse set of applications some of which will be mentioned shortly. 
 Many versions and generalizations of the multi-armed bandit problem have been studied in the literature; in this paper we will consider a basic
and well-studied version of this problem: the stochastic multi-armed bandit problem. 
 Among many algorithms available for the stochastic bandit problem, some popular ones include Upper Confidence Bound (UCB) family of algorithms, (e.g., \cite{LaiR, Aueretal}, and more 
recently \cite{Garivier11}, \cite{Maillard11}, \cite{Kauffman12}), which have good theoretical guarantees, and the algorithm by \cite{Gittins}, which gives optimal strategy under Bayesian setting with known priors and geometric time-discounted rewards. 
%%TODO: Check 
 In one of the earliest works on stochastic bandit problems, \cite{Thompson} proposed a natural randomized Bayesian algorithm to minimize regret. 
  The basic idea is to assume a simple prior distribution on the parameters of the reward distribution of every arm, and at any time step, play an arm according to its posterior probability of being the best arm. This algorithm is known as \emph{Thompson Sampling} (TS), and it is a member of the family of \emph{randomized probability matching} algorithms. 
 We emphasize that although TS algorithm is a Bayesian approach, the description of the algorithm and our analysis apply to the prior-free stochastic multi-armed bandit model where parameters of the reward distribution of every arm are fixed, though unknown (refer to Section \ref{sec:MAB}). One could think of the ``assumed" Bayesian priors as a tool employed by the TS algorithm to encode the current knowledge about the arms. 
Thus, our regret bounds for Thompson Sampling are directly comparable to the regret bounds for UCB family of algorithms which are a frequentist approach to the same problem. 
 
Recently, TS has attracted considerable attention. Several studies (e.g., \cite{Granmo, Scott, CLi, MayL}) have empirically demonstrated the efficacy of Thompson Sampling: \cite{Scott} provides a detailed discussion of probability matching techniques in many general settings along with favorable empirical comparisons with other techniques. \cite{CLi} demonstrate that empirically TS achieves regret comparable to the lower bound of \cite{LaiR}; and in applications like display advertising and news article recommendation, it is competitive to or better than popular methods such as UCB. In their experiments, TS is also more robust to delayed or batched feedback (delayed feedback means that the result of a play of an arm may become available only after some time delay, but we are required to make immediate decisions for which arm to play next) than the other methods. 
%%% Review the sentence below, 
A possible explanation may be that TS is a randomized algorithm and so it is unlikely to 
get trapped in an early bad decision during the delay. 
Microsoft's adPredictor (\cite{GraepelCBH10}) for CTR prediction of search ads on Bing uses the idea of Thompson Sampling. 
%%%%This needs further verfication before we put this in: TS generalizes to many settings where other algorithms like
%%$%UCB-type algorithms might not generalize. Computationally efficient
%%$%because of the availability of efficient Bayesian sampling methods. 

It has been suggested (\cite{CLi}) that despite being easy to implement and being competitive to the state of the art methods, the reason TS is not very popular in literature could be its lack of strong theoretical analysis. Existing theoretical analyses in \cite{Granmo, MayKLL} provide weak guarantees, namely, a bound of $o(T)$ on expected regret in time $T$. 
%Also, these works provide only asymptotic bounds, and do not bound the finite-time expected regret.  
In this paper, for the first time, we provide a logarithmic bound on expected regret of TS algorithm in time $T$ that is close to the lower bound of~\cite{LaiR}. 
%% Further, we show that TS achieves logarithmic regret uniformly over time, rather than only asymptotically.
Before stating our results, we describe the MAB problem and the TS algorithm formally. 

%only asymptotic convergence and do not provide bounds on the finite-time regret which is of primary interest. 
%In this paper, for the first time, we provide a logarithmic bound on expected regret in time $T$ that is  close to the lower bound of~\cite{LaiR}. Before stating our results we describe the MAB problem and the TS alorithm formally. 
% present the first theoretical analysis of finite-time expected regret achieved by TS, and show that it is close to the lower bound of~\cite{LaiR}. 

%May~et~al.~\cite{MayKLL} present Optimistic Thompson Sampling algorithm (OTS), which is closely related to TS. \cite{MayL, CLi} show empirically that OTS has slightly smaller regret than TS, though the improvement is rather small.
%\cite{CLi} suggested that despite being competitive to the state of the art methods, TS was not popular because no theoretical guarantees were not known for it. 

\subsection{The multi-armed bandit problem}
\label{sec:MAB}
We consider the stochastic multi-armed bandit (MAB) problem: We are given a slot machine with $N$ arms; 
at each time step $t=1, 2, 3, \ldots $, one of the $N$ arms must be chosen to be played. 
Each arm $i$, when played, yields a random real-valued reward according to some fixed (unknown) distribution with support in $[0,1]$. 
% : In each step we choose one of the arms to play we play arm $i$ we get a reward. 
%The reward for arm $i$ is a random variable with mean $\mu_i$. In each play of an arm, 
The random reward obtained from playing an arm repeatedly are i.i.d. and independent of the plays of the other arms. 
The reward is observed immediately after playing the arm. 
%In this paper, we will make a further assumption that the reward distributions are \emph{Bernoulli}, i.e., the rewards are $0$ or $1$, and for arm $i$ the probability of success (reward =$1$) is $\mu_i$. This is perhaps the most studied model that while being simple captures the essence of the problem and is practically relevant as well. 

An algorithm for the MAB problem must decide which arm to play at each time step $t$, based on the outcomes of the previous $t-1$ plays. 
Let $\mu_i$ denote the (unknown) expected reward for arm $i$.
%The means $\mu_1, \mu_2, \ldots, \mu_N$ are unknown, and we must learn about them by playing the corresponding arms.
A popular goal is to maximize the expected total reward in time $T$, i.e., $\Ex[\sum_{t=1}^{T} \mu_{i(t)}]$, where $i(t)$ is the arm played 
in step $t$, and the expectation is over the random choices of $i(t)$ made by the algorithm. 
%based on the plays and random outcomes until time $t-1$. 
It is more convenient
to work with the equivalent measure of expected total \emph{regret}: the amount we lose because of not playing optimal arm in each step. 
To formally define regret, let us introduce some notation. 
%Denote the best or optimal arm by $i^*$, i.e., $i^* := \arg \max_{i} \mu_i$, and $\mu^*:=\mu_{i^*}$ (if there are more than one optimal arms, pick
Let $\mu^*:=\max_i \mu_i$, and $\Delta_i := \mu^* - \mu_i$. Also, let $k_i(t)$ denote the number of times arm $i$ has been played up to step $t-1$.  Then the
expected total regret in time $T$ is given by \vspace{-0.05in}
$$\mbox{$\ExR{{\cal R}(T)}= \ExR{\sum_{t=1}^{T} (\mu^{*}-\mu_{i(t)})} = \sum_{i} \Delta_i \cdot \ExR{k_i(T)}. $}$$
%Thus our goal is to 
%minimize the expected
%total regret. This is perhaps the most popular way to characterize the performance of MAB algorithms. 
Other performance measures include PAC-style guarantees; we do not consider those measures here. 
%%%%Need to say whether T is known in advance or not. I think the algorithms don't use T, so we don't need to assume
%%%%it's known. 

%% In this paper, we provide bounds on the expected total regret of Thompson Sampling algorithm for the stochastic multi-armed bandit problem. We will furthermore assume \emph{Bernoulli} bandits, that is, each arm $i$ generates a random $0-1$ reward, with $\mu_i$ being the probability of getting reward $1$.\vspace{-0.1in}

\subsection{Thompson Sampling}
%%Need to mention that we are focusing only on Bernoulli, and why
%In the most general setting, Thompson Sampling can be described as a natural Bayesian algorithm that assumes a prior distribution on the parameters of the reward distributions, and at any time step, plays an arm according to its probability of being the best arm.
For simplicity of discussion, we first provide the details of Thompson Sampling algorithm for the Bernoulli bandit problem, i.e. when the rewards are either $0$ or $1$, and for arm $i$ the probability of success (reward =$1$) is $\mu_i$. This description of Thompson Sampling follows closely that of \cite{CLi}. Next, we propose a simple new extension of this algorithm to general reward distributions with support $[0,1]$, which will allow us to seamlessly extend our analysis for Bernoulli bandits to general stochastic bandit problem.

The algorithm for Bernoulli bandits maintains Bayesian priors on the Bernoulli means $\mu_i$'s. Beta distribution turns out to be a very convenient 
choice of priors for Bernoulli rewards. 
Let us briefly recall that beta distributions form a family of continuous probability distributions on the interval $(0,1)$. The pdf of $\Beta(\alpha, \beta)$, the beta distribution with parameters $\alpha > 0$, $\beta > 0$, is given by $f(x; \alpha, \beta) = \frac{\Gamma(\alpha+\beta)}{\Gamma(\alpha)\Gamma(\beta)}x^{\alpha-1}(1-x)^{\beta-1}$. The mean of $\Beta(\alpha, \beta)$ is $\alpha/(\alpha+\beta)$; and as is apparent from the pdf, higher the $\alpha, \beta$, tighter is the concentration of $\Beta(\alpha, \beta)$ around the mean. 
Beta distribution is useful for Bernoulli rewards 
because if the prior is a $\Beta(\alpha, \beta)$ distribution, then after observing a Bernoulli trial, the posterior distribution is simply $\Beta(\alpha+1, \beta)$ or $\Beta(\alpha, \beta+1)$, depending on whether the trial resulted in a success or failure, respectively. 

The Thompson Sampling algorithm initially assumes arm $i$ to have prior $\Beta(1, 1)$ on $\mu_i$, which is
natural because $\Beta(1,1)$ is the uniform distribution on $(0,1)$. At time $t$, having observed $S_i(t)$ successes (reward = $1$) and $F_i(t)$  failures (reward = $0$) in $k_i(t) = S_i(t)+F_i(t)$ plays of arm $i$, the algorithm updates the distribution on $\mu_i$ as $\Beta(S_i(t)+1, F_i(t)+1)$. 
%Note that the beta distribution with parameters $S_i(t)+1, F_i(t)+1$ is the posterior distribution of $\mu_i$ after observing $S_i(t)$ successes (with probability $\mu_i$ of success) and $F_i(t)$ failures (with probability $1 - \mu_i$ of failure). 
The algorithm then samples from these posterior distributions of the $\mu_i$'s, and plays an arm according to the probability of its mean being the largest. We summarize the Thompson Sampling algorithm below. \vspace{-0.05in}
  %% In the $N$-armed Bernoulli bandit problem, each arm yields a reward of $0$ or $1$, with $\mu_1, \ldots, \mu_N$ being the respective probabilities of getting a reward of $1$. Denote the best arm by $i^*$, i.e., $i^* = \arg \max_{i} \mu_i$, and $\mu^*=\mu_{i^*}$ as defined earlier. Also, set $\Delta_i := \mu^* - \mu_i$. 
%%  At every time step $t$, one of the arms is selected to be played by  algorithm as follows. 
%Let $S_i(t), F_i(t)$ denote the number of successes (reward = $1$) and the number of failures (reward = $0$), respectively, observed for arm $i$ until time $t-1$. 
%Note that $k_i(t) = S_i(t) + F_i(t)$.

  \begin{algorithm}[H]
  \caption{Thompson Sampling for Bernoulli bandits}
$S_i=0, F_i=0$.\\
\ForEach{$t=1, 2, \ldots,$}{
		For each arm $i=1,\ldots, N$, sample $\theta_i(t)$ from the $\Beta(S_i+1, F_i+1)$ distribution. \\
  	Play arm $i(t) := \arg \max_i \theta_i(t)$ and observe reward $r_t$.\\
  	If $r=1$, then $S_{i}=S_i+1$, else $F_i=F_i+1$.
  	}
\end{algorithm}

% For general (as opposed to Bernoulli) stochastic bandits, the rewards for arm $i$ are generated from an arbitrary unknown distribution with support $[0,1]$ and mean $\mu_i$. Thompson sampling generalizes very naturally to this general setting: The structure of the algorithm remains the same except for the updates. Suppose that the prior for arm $i$ at a step $t$ is $\Beta(\alpha_i, \beta_i)$, and that arm is played in step $t$ resulting in reward $\tilde{r_t}$. Then we update the prior to $\Beta(\alpha_i + \tilde{r_t}, \beta_i + 1 - \tilde{r_t})$. Note that this is a direct generalization of Bernoulli updates. We expect that this general algorithm has similar performance as the one for Bernoulli bandits; however, its analysis leads to additional complications: e.g., our analysis makes a crucial use of a relation between the Beta distribution with integerl parameters  and the binomial distribution (Fact~\ref{fact:beta-binomial}), and we do not know a corresponding generalization for Beta distribution with fractional parameters. 
We adapt the Bernoulli Thompson sampling algorithm to the general stochastic bandits case, i.e. when the rewards for arm $i$ are generated from an arbitrary unknown distribution with support $[0,1]$ and mean $\mu_i$, in a way that allows us to reuse our analysis of the Bernoulli case. To our 
knowledge, this adaptation is new. 
%% Let $\tilde{r_t} \in [0,1]$ denote the reward observed at time $t$.
We modify TS so that after observing the reward $\tilde{r}_t\in [0,1]$ at time $t$, it performs a Bernoulli trial with success probability $\tilde{r_t}$. Let random variable $r_t$ denote the outcome of this Bernoulli trial, and let $\{S_i(t), F_i(t)\}$ denote the number of successes and failures in the Bernoulli trials until time $t$. The remaining algorithm is the same as for Bernoulli bandits. Algorithm 2 gives the precise description of this algorithm.
% Thus, Thompson Sampling generates $\theta_i(t)$ from the posterior distribution of $\mu_i$, given the observed outcomes until time $t-1$. The arms are played according to their probability of having the greatest $\mu_i$.

We observe that the probability of observing a success (i.e., $r_t=1$) in the Bernoulli trial after playing an arm $i$ in the new generalized algorithm is equal to the mean reward $\mu_i$. Let $f_i$ denote the (unknown) pdf of reward distribution for arm $i$. Then, on playing arm $i$,
  $$\mbox{$\Pr(r_t=1) = \int_0^1 \tilde{r} f_i(\tilde{r}) d\tilde{r} = \mu_i.$}$$
 Thus, the probability of observing $r_t=1$ is same and $S_i(t), F_i(t)$ evolve exactly in the same way as in the case of Bernoulli bandits with mean $\mu_i$. Therefore, the analysis of TS for Bernoulli setting
is applicable to this modified TS for the general setting. 
This allows us to replace, for the purpose of analysis, the problem with general stochastic bandits with Bernoulli bandits with the same means. We use this observation to confine the proofs in this paper to the case of Bernoulli bandits only. 

\begin{algorithm}[h]
 \caption{Thompson Sampling for general stochastic bandits}
   $S_i=0, F_i=0$.\\
   \ForEach{$t=1, 2, \ldots,$}{
        For each arm $i=1,\ldots, N$, sample $\theta_i(t)$ from the $\Beta(S_i+1, F_i+1)$ distribution. \\
  			Play arm $i(t) := \arg \max_i \theta_i(t)$ and observe reward $\tilde{r}_t$.\\
  			{\bf Perform a Bernoulli trial with success probability $\tilde{r}_t$ and observe output $r_t$.}\\
  			If $r_t=1$, then $S_{i}=S_i+1$, else $F_i=F_i+1$.
   	}
  \end{algorithm}
%\begin{center}
% \fbox{
%  \begin{minipage}[h]{6in}
%  Set $S_i=0, F_i=0$.\\
  
%  At time step $t=1, 2, \ldots,$, \\
%  		For each arm $i=1,\ldots, N$, generate $\theta_i(t)$ from the beta distribution with parameters $(S_i+1, F_i+1)$. \\
%  		Play arm $i(t) := \arg \max_i \theta_i(t)$ and observe reward $\tilde{r}$.\\
%  		{\bf Perform Bernoulli trial with success probability $\tilde{r}$ and observe output $r$.}\\
%  		if $r=1$, then $S_{i}=S_i+1$, else $F_i=F_i+1$.
%  \end{minipage}
% }
% \end{center}

\subsection{Our results}
In this article, we bound the \emph{finite time} expected regret of Thompson Sampling. 
From now on we will assume that the first arm is the unique optimal arm, i.e., $\mu^* = \mu_1 > \arg \max_{i\ne 1} \mu_i$. Assuming that the first arm is an optimal arm is a matter of convenience for stating the results and for the analysis. 
The assumption of \emph{unique} optimal arm is also without loss of generality, since adding more arms with $\mu_i=\mu^*$ can only decrease the expected regret; details of this argument are provided in Appendix \ref{app:mul-opt}.

\begin{theorem}
\label{th:two-arms}
For the two-armed stochastic bandit problem ($N=2$), Thompson Sampling algorithm has expected regret  \vspace{-0.1in}
$$\Ex[{\cal R}(T)] = O\left(\frac{\ln T}{\Delta}+\frac{1}{\Delta^3}\right) \vspace{-0.1in}$$
in time $T$, where $\Delta = \mu_1-\mu_2$.
\end{theorem}
\begin{theorem}
\label{th:N-arms}
For the $N$-armed stochastic bandit problem, Thompson Sampling algorithm has expected regret \vspace{-0.1in}
$$\Ex[{\cal R}(T)] \leq O\left(\left(\sum_{a=2}^N\frac{1}{\Delta_a^2}\right)^2 \ln T\right)  \vspace{-0.1in}$$
in time $T$, where $\Delta_i=\mu_1-\mu_i$.
\end{theorem}
\begin{remark}
\label{rem:N-arms}
For the $N$-armed bandit problem, we can obtain an alternate bound of \vspace{-0.05in}
$$\Ex[{\cal R}(T)] \leq O\left(\frac{\Delta_{max}}{\Delta_{min}^3}\left(\sum_{a=2}^N\frac{1}{\Delta_a^2}\right) \ln T\right) \vspace{-0.05in}$$
by slight modification to the proof. The above bound has a better dependence on $N$ than in Theorem \ref{th:N-arms}, but worse dependence on $\Delta_{i}s$. Here $\Delta_{min}=\min_{i\ne 1} \Delta_i$,$\Delta_{max}=\max_{i\ne 1} \Delta_i$.
\end{remark}
In interest of readability, we used big-Oh notation \footnote{For any two functions $f(n),g(n)$, $f(n)=O(g(n))$ if there exist two constants $n_0$ and $c$ such that for all $n\ge n_0$, $f(n) \le cg(n)$.} to state our results. The exact constants are provided in the proofs of the above theorems.
%We remark that we have not attempted to optimize constants in the above theorems in the interest of 
%readability. 
%We may also remark that if there are arms other than the first arm with $\Delta_i=0$, then those arms can be ignored for the purpose of regret analysis; i.e., in case of multiple optimal arms the summation in the above theorem statement will be only over arms with non-zero $\Delta_i$s. The idea is that adding another \emph{optimal} arm to the bandit can only decrease the expected regret of the algorithm. Thus, the case of a unique optimal arm is in fact the worst case for regret analysis. The details of this argument appear in Appendix \ref{app:mul-opt}.
Let us contrast our bounds with the previous work. \cite{LaiR} proved the following lower bound on regret of any bandit algorithm: 
$$\Ex[{\cal R}(T)] \geq \left[\sum_{i=2}^{N} \frac{\Delta_i}{D(\mu_i || \mu)} + o(1) \right] \ln{T}, $$
where $D$ denotes the KL divergence. They also gave algorithms asymptotically achieving this guarantee, though unfortunately their algorithms are not efficient. \cite{Aueretal} gave the UCB1 algorithm, which is efficient and achieves the following bound: 
$$\Ex[{\cal R}(T)] \le \left[8 \sum_{i=2}^N \frac{1}{\Delta_i} \right] \ln{T} + (1+\pi^2/3)\left(\sum_{i=2}^{N} \Delta_i\right).$$
For many settings of the parameters, the bound of Auer~et~al. is not far from the lower bound of Lai and Robbins. 
Our bounds are optimal in terms of dependence on $T$, but inferior in terms of the constant factors and dependence on $\Delta$. 
We note that for the two-armed case our bound closely matches the bound of  \cite{Aueretal}. %% While the above theorem is stated in order notation of simplicity of exposition, the exact constants have been derived in the text. 
For the $N$-armed setting, the exponent of $\Delta$'s in our bound is basically $4$ compared to the exponent $1$ for UCB1. 

More recently, \cite{Kauffman12} gave Bayes-UCB algorithm which achieves regret bounds close to the lower bound of \cite{LaiR} for Bernoulli rewards. Bayes-UCB is a UCB like algorithm, where the upper confidence bounds are based on the quantiles of Beta posterior distributions. Interestingly, these upper confidence bounds turn out to be similar to those used by algorithms in \cite{Garivier11} and \cite{Maillard11}. 
Bayes-UCB can be seen as an hybrid of TS and UCB. However, the general structure of the arguments used in \cite{Kauffman12} is similar to \cite{Aueretal}; for the analysis of Thompson Sampling we need to deal with additional difficulties, as discussed in the next section.

\section{Proof Techniques} 
\label{sec:techniques}
In this section, we give an informal description of the techniques involved in our analysis. We hope that this will aid in reading the proofs, though this section is not essential for the sequel. We assume that all arms are Bernoulli arms, and that the first arm is the unique optimal arm. 
As explained in the previous sections, these assumptions are without loss of generality.

\paragraph{Main technical difficulties.}
Thompson Sampling is a randomized algorithm which achieves exploration by choosing to play the arm with best sampled mean, among those generated from beta distributions around the respective empirical means. The beta distribution becomes more and more concentrated around the empirical mean as the number of plays of an arm increases. This randomized setting is unlike the algorithms in UCB family, which achieve exploration by adding a \emph{deterministic, non-negative} bias inversely proportional to the number of plays, to the observed empirical means. %% Owing to the fact that TS adds a random bias whose distribution changes with time, and which could even be negative, many new ideas and techniques were required to analyze this algorithm. 
Analysis of TS poses difficulties that seem to require new ideas.

For example, following general line of reasoning is used to analyze regret of UCB like algorithms in two-arms setting (for example, in \cite{Aueretal}): once the second arm has been played sufficient number of times, its empirical mean is tightly concentrated around its actual mean. If the first arm has been played sufficiently large number of times by then, it will have an empirical mean close to its actual mean and larger than that of the second arm. Otherwise, if it has been played small number of times, its non-negative bias term will be large. Consequently, once the second arm has been played sufficient number of times, it will be played with very small probability (inverse polynomial of time) \emph{regardless of the number of times the first arm has been played so far}. 

However, for Thompson Sampling, if the number of previous plays of the first arm is small, then the probability of playing the second arm could be as large as a constant even if it has already been played large number of times. For instance, if the first arm has not been played at all, then $\theta_1(t)$ is a uniform random variable, and thus $\theta_1(t) < \theta_2(t)$ with probability $\theta_2(t) \approx \mu_2$. As a result, in our analysis we need to carefully consider the distribution of the number of previous plays of the first arm, in order to bound the probability of playing the second arm. 
%In fact, analyzing the case of small number of plays of first arm forms the main difficulty in our analysis of the two-arms setting.

The observation just mentioned also points to a challenge in extending the analysis of TS for two-armed bandit to the general $N$-armed bandit setting. One might consider analyzing the regret in the $N$-armed case by considering only two arms at a time---the first arm and one of the suboptimal arms. We could use the observation that the probability of playing a suboptimal arm is bounded by the probability of it exceeding the first arm. However, this probability also depends on the number of previous plays of the two arms, which in turn depend on the plays of the other arms. Again, \cite{Aueretal}, in their analysis of UCB algorithm, overcome this difficulty by bounding this probability for \emph{all possible numbers of previous plays} of the first arm, and large enough plays of the suboptimal arm. For Thompson Sampling, due to the observation made earlier, the (distribution of the) number of previous plays of the first arm needs to be carefully accounted for, which in turn requires considering all the arms at the same time, thereby leading to a more involved analysis.

%We bound the regret by bounding the number of plays of the second arm. 
\paragraph{Proof outline for two arms setting.} Let us first consider the special case of two arms which is simpler than the general $N$ arms case. 
Firstly, we note that it is sufficient to bound the regret incurred during the time steps \emph{after} the second arm has been played \remove{$L=16(\ln T)/\Delta^2$} \add{$L=24(\ln T)/\Delta^2$} times. The expected regret before this event is bounded by $24(\ln T)/\Delta$ because only the plays of the second arm produce an expected regret of $\Delta$; regret is $0$ when the first arm is played. 
%the steps involving the play of the second arm produce expected regret $\Delta$; steps in which the first arm is played have expected regret $0$.
Next, we observe that after the second arm has been played $L$ times, the following happens with high probability: the empirical average reward of the second arm from each play is very close to its actual expected reward $\mu_2$, and its beta distribution is tightly concentrated around $\mu_2$. This means that, thereafter, the first arm would be played at time $t$ if $\theta_1(t)$ turns out to be greater than (roughly) $\mu_2$. This observation allows us to model the number of steps between two consecutive plays of the first arm as a geometric random variable with parameter close to $\Pr[\theta_1(t) > \mu_2]$. 
%We obtain our bound by carefully bounding the expected value of this geometric random variable for possible distributions of $\theta_1(t)$.
To be more precise, given that there have been $j$ plays of the first arm with $s(j)$ successes and $f(j)=j-s(j)$ failures, we want to estimate the expected number of steps before the first arm is played again (not including the steps in which the first arm is played). This is modeled by a geometric random variable $X(j,s(j), \mu_2)$ with parameter $\Pr[\theta_1 > \mu_2]$, where $\theta_1$ has distribution $\mbox{Beta}(s(j)+1, j-s(j)+1)$, and thus $\ExB{X(j,s(j), \mu_2)\condB s(j)} = 1/\Pr[\theta_1 > \mu_2] - 1$. To bound the overall expected number of steps between the $j^{th}$ and  $(j+1)^{th}$ play of the first arm, we need to take into account the distribution of the number of successes $s(j)$. For large $j$, 
%similar to the observation made earlier for $\theta_2$, the behavior of random variable $s(j)$ can be understood using 
we use Chernoff--Hoeffding bounds to say that $s(j)/j \approx \mu_1$ with high probability, and moreover $\theta_1$ is concentrated around its mean, and thus we get a good estimate of $\ExR{\ExB{X(j,s(j),\mu_2) \condB s(j)}}$. However, for small $j$ we do not have such concentration, and it requires a delicate computation to get a bound on $\ExR{\ExB{X(j,s(j), \mu_2)\condB s(j)}}$. The resulting bound on the expected number of steps between consecutive plays of the first arm bounds the expected number of plays of the second arm, to yield a good bound on the regret for the two-arms setting. 
%This forms the first part of our proof.

\paragraph{Proof outline for $N$ arms setting.} 
%In this case, we cannot assume that all the suboptimal arms have been tried enough to have their sample values concentrated around their means (that may require many more than $\sum_i \ln T/\Delta_i^2$ steps).  
%Instead, 
 At any step $t$, we divide the set of suboptimal arms into two subsets: \emph{saturated} and \emph{unsaturated}. The set $C(t)$ of saturated arms at time $t$ consists of arms $a$ that have already been played a sufficient number ($L_a = 24(\ln T)/\Delta^2_a$) of times, so that with high probability, $\theta_a(t)$ is tightly concentrated around $\mu_a$.
%Intuitively, the behavior of saturated arms is more predictable, while the unsaturated arms may demonstrate high variablity in their parameters.
 As earlier, we try to estimate the number of steps between two consecutive plays of the first arm. After $j^{th}$ play, the $(j+1)^{th}$ play of first arm will occur at the earliest time $t$ such that $\theta_1(t) > \theta_i(t), \forall i\ne 1$. The number of steps before $\theta_1(t)$ is greater than $\theta_a(t)$ of all saturated arms $a \in C(t)$ can be \remove{analyzed} \add{closely approximated} using a geometric random variable with parameter close to $\Pr(\theta_1 \ge \max_{a \in C(t)} \mu_a)$, as before. However, even if $\theta_1(t)$ is greater than the $\theta_a(t)$ of all saturated arms $a\in C(t)$, it may not get played due to play of an unsaturated arm $u$ with a greater $\theta_u(t)$. Call this event an ``interruption" by unsaturated arms. We show that if there have been $j$ plays of first arm with $s(j)$  successes, the expected number of steps until the $(j+1)^{th}$ play can be upper bounded by the product of the expected value of a geometric random variable similar to $X(j,s(j), \max_a \mu_a)$ defined earlier, and the number of interruptions by the unsaturated arms. Now, the total number of interruptions by unsaturated arms is bounded by $\sum_{u=2}^N L_u$ (since an arm $u$ becomes saturated after $L_u$ plays). The actual number of interruptions  is hard to analyze due to the high variability in the parameters of the unsaturated arms. We derive our bound assuming the worst case allocation of these $\sum_u L_u$ interruptions. This step in the analysis is the main source of the high 
exponent of $\Delta$ in our regret bound for the $N$-armed case compared to the two-armed case.

\section{Regret bound for the two-armed bandit problem}
In this section, we present a proof of Theorem \ref{th:two-arms}, our result for the two-armed bandit problem. %, while omitting many technical details due to space considerations. 
Recall our assumption that all arms have Bernoulli distribution on rewards, and that the first arm is the unique optimal arm. 

Let random variable $j_0$ denote the number of plays of the first arm until \remove{$L=16(\ln T)/\Delta^2$} \add{$L=24(\ln T)/\Delta^2$} plays of the second arm. \add{Let random variable $t_j$ denote the time step at which the $j^{th}$ play of the first arm happens (we define $t_0=0$).} Also, let random variable $Y_j\add{=t_{j+1}-t_j-1}$ measure the number of time steps between the $j^{th}$ and $(j+1)^{th}$ plays of the first arm (not counting the steps in which
the $j^{th}$ and $(j+1)^{th}$ plays happened), and let $s(j)$ denote the number of successes in the first $j$ plays of the first arm. Then the expected number of plays of the second arm in time $T$ is bounded by 
\begin{center}
$\Ex[k_2(T)] \le L + \ExR{\sum_{j=j_0}^{T-1} Y_j}.$
\end{center} %= \frac{4\ln T}{\Delta^2} + \Ex[\sum_{j=j_0}^{T} \Ex[Y_j|s(j), j_0]]$$

To understand the expectation of $Y_j$, it will be useful to define another random variable
$X(j,s,y)$ as follows. We perform the following experiment until it succeeds: check if a $\mbox{Beta}(s+1, j-s+1)$ distributed random variable exceeds a threshold $y$. For each experiment, we generate the beta-distributed r.v. independently of the previous ones.  
Now define $X(j,s,y)$ to be the number of trials \emph{before} 
the experiment succeeds. Thus, $X(j,s,y)$ takes non-negative integer values, and is a geometric random variable with parameter (success probability) $1-F^{beta}_{s+1,j-s+1}(y)$. Here $F^{beta}_{\alpha, \beta}$ denotes the cdf of the beta distribution with parameters $\alpha, \beta$. Also, let $F^B_{n,p}$ denote the cdf of the \emph{binomial} distribution with parameters $(n,p)$.

We will relate $Y$ and $X$ shortly. The following lemma provides a handle on the expectation of $X$. 
\begin{lemma}
\label{lem:exj-1}
For all non-negative integers $j,s\le j$, and for all $y\in [0,1]$, \vspace{-0.05in}
$$\mbox{$\ExR{X(j,s,y)} = \frac{1}{F^B_{j+1,y}(s)} - 1$},\vspace{-0.05in}$$
where $F^B_{n,p}$ denotes the cdf of the binomial distribution with parameters $(n,p)$.
\end{lemma}
\begin{proof}
By the well-known formula for the expectation of a geometric random variable and the definition of $X$ we have, 
$\ExR{X(j,s,y)} = \frac{1}{1-F^{beta}_{s+1, j-s+1}(y)}-1$
%where $F^{beta}_{\alpha, \beta}$ denotes cdf of the beta distribution with parameters $(\alpha,\beta)$. 
(The additive $-1$ is there because we do not count the final step where the Beta r.v. is greater than $y$.) 
The lemma then follows from Fact~\ref{fact:beta-binomial} in Appendix \ref{app:facts}. 
\end{proof}

\remove{In order to bound expected value of $Y$ using expected value of $X$, 
we will condition on an event $E_2(T)$, which we define to hold iff $\theta_2(t) \le \mu_2 + \frac{\Delta}{2}$ for all time $t \in [1, T]$ such that the second arm has already been played at least $L$ times before time $t$. The following lemma can be proven 
using Fact~\ref{fact:beta-binomial} and the Chernoff--Hoeffding bounds. 
\begin{lemma}
\label{lem:E2}
\forall $t$, $ \Pr(E_2(T))\ge 1-\frac{2}{T}.$
\end{lemma}
\begin{proof}
Refer to Appendix \ref{app:lem:E2}.
\end{proof}
Recall that $Y_j$ was defined as the number of steps before $\theta_1(t) > \theta_2(t)$ for the first time after the $j^{th}$ play of the first arm. Now, under event $E_2(T)$, $\theta_2(t) \le \mu_2+\Delta/2$, for all time $t$ after $L$ plays of the second arm. Therefore, given that event $E_2(T)$ holds, and that the number of successes in $j$ trials of first arm is $s(j)$, $Y_j$ for $j\ge j_0$ is  stochastically dominated by  geometric random variable $X(j,s(j),\mu_2+\frac{\Delta}{2})$. Also, it is bounded by $T$. If the event $E_2(T)$ does not hold, we can bound $\sum_j Y_j$ by the trivial upper bound of $T$. Using these observations, and the bound on $\Pr(E_2(T))$ from Lemma \ref{lem:E2},%\vspace{-0.2in}
\begin{eqnarray*}
\mbox{$\ExR{\sum_{j=j_0}^{T-1} Y_j}$} & \le & \mbox{$\ExR{ \sum_{j=j_0}^{T-1} \ExB{ Y_j \condB j_0, s(j), E_2(T) }} + \overline{\Pr(E_2(T))} \cdot T$} \\
%& \le & \Ex[\sum_{j=j_0}^{T-1}  \min\{ \Ex[X(j,s(j),\mu_2+\frac{\Delta}{2}) \ | \ s(j)] ,  T\}] + \frac{2}{T} \cdot T \\
 & \le & \mbox{$\ExR{\sum_{j=0}^{T-1}  \min\{ \ExB{X(j,s(j),\mu_2+\frac{\Delta}{2}) \condB s(j)} ,  T\}} + 2$}. 
\end{eqnarray*}
 }
\add{Recall that $Y_j$ was defined as the number of steps before $\theta_1(t) > \theta_2(t)$ happens for the first time after the $j^{th}$ play of the first arm. Now, consider the number of steps before $\theta_1(t) > \mu_2+\frac{\Delta}{2}$ happens for the first time after the $j^{th}$ play of the first arm. Given $s(j)$, this has the same distribution as $X(j,s(j), \mu_2 + \frac{\Delta}{2})$. However, $Y_j$ can be larger than this number if (and only if) at some time step $t$ between $t_{j}$ and $t_{j+1}$, $\theta_2(t) > \mu_2+\frac{\Delta}{2}$. In that case we use the fact that  $Y_j$ is always bounded by $T$. Thus, for any $j\ge j_0$, we can bound $\Ex[Y_j]$ as,
$$\mbox{$\Ex[Y_j] \le \Ex[ \min\{X(j,s(j), \mu_2+\frac{\Delta}{2}), T\}] + \Ex[\sum_{t=t_{j}+1}^{t_{j+1-1}} T \cdot I(\theta_2(t) > \mu_2 + \frac{\Delta}{2})].$}$$
Here notation $I(E)$ is the indicator for event $E$, i.e., its value is $1$ if event $E$ happens and $0$ otherwise. In the first term of RHS, the expectation is over distribution of $s(j)$ as well as over the distribution of the geometric variable $X(j,s(j),\mu_2+\frac{\Delta}{2})$.
Since we are interested only in $j\ge j_0$, we will instead use the similarly obtained bound on $\Ex[Y_j \cdot I(j\ge j_0)]$,
$$\mbox{$\Ex[Y_j \cdot I(j\ge j_0)] \le \Ex[ \min\{X(j,s(j), \mu_2+\frac{\Delta}{2}), T\}] + \Ex[\sum_{t=t_{j}+1}^{t_{j+1-1}} T \cdot I(\theta_2(t) > \mu_2 + \frac{\Delta}{2})\cdot I(j\ge j_0)].$}$$
This gives,
\begin{eqnarray*}
\mbox{$\Ex[\sum_{j=j_0}^{T-1} Y_j]$} & \le & \mbox{$\sum_{j=0}^{T-1} \Ex[ \min\{X(j,s(j), \mu_2+\frac{\Delta}{2}), T\}] + T \cdot \sum_{j=0}^{T-1} \Ex[\sum_{t=t_{j}+1}^{t_{j+1-1}} I(\theta_2(t) > \mu_2 + \frac{\Delta}{2}, j\ge j_0)]$}\\
& \le & \mbox{$\sum_{j=0}^{T-1} \Ex[ \min\{X(j,s(j), \mu_2+\frac{\Delta}{2}), T\}]  + T \cdot \sum_{t=1}^T \Pr(\theta_2(t) > \mu_2 + \frac{\Delta}{2}, k_2(t)\ge L)$}.
\end{eqnarray*}
 The last inequality holds because for any $t\in [t_j+1, t_{j+1}-1], j\ge j_0$, by definition $k_2(t) \ge L$. 
 We denote the event  $\{\theta_2(t) \leq \mu_2 + \frac{\Delta}{2} \mbox{ or } k_2(t) < L\}$ by $E_2(t)$. In words, this is
the event that if sufficient number of plays of second arm have happened until time $t$, then $\theta_2(t)$ is not much larger than $\mu_2$; intuitively, we expect this event
to be a high probability event as we will show.
$\overline{E_2(t)}$ is the event $\{\theta_2(t) > \mu_2 + \frac{\Delta}{2} \mbox{ and } k_2(t)\ge L\}$ used in the above equation. Next, we bound 
$\Pr(E_2(t))$ and $\Ex[ \min\{X(j,s(j), \mu_2+\frac{\Delta}{2}), T\}]$.
\begin{lemma}
\label{lem:E2}
$$\forall t, \ \ \ \Pr(E_2(t)) \ge 1-\frac{2}{T^2}.$$
\end{lemma}
\begin{proof}
Refer to Appendix \ref{app:lem:E2}.
\end{proof}
}
%Below, we upper bound the expected value $\Ex[\Ex[X_{j,s(j),}]]$, for all $j$, for any given value of $y =\max_{i\ne 1} \hat{\theta}_i$. We state this result in more general terms so that it can be reused in the next part of the proof.
\begin{lemma}
\label{lem:exj}
%With probability $f^B_{j,\mu_1}(s)$, $X_{j}$ is a geometric random variable with success probability $1-F^{beta}_{s+1,j-s+1}(y)$.
Consider any positive $y < \mu_1$, and let $\Delta' = \mu_1 - y$. 
Also, let $R=\frac{\mu_1(1-y)}{y(1-\mu_1)} > 1$, and let $D$ denote the KL-divergence between $\mu_1$ and $y$, i.e.
 $D= y \ln \frac{y}{\mu_1} + (1-y) \ln \frac{1-y}{1-\mu_1}.$
%Note that $D\ge 0$. 
\begin{eqnarray*}
\ExR{\ExB{ \min\{ \ X(j, s(j), y) , \ T \}\condB s(j)} } & \le & \ \left\{\begin{array}{ll}
   \displaystyle  1+ \frac{2}{1-y} + \frac{\mu_1}{\Delta'} e^{-Dj} &  j < \frac{y}{D}\ln R,\\
   \displaystyle 1+ \frac{R^y}{1-y}e^{-Dj} + \frac{\mu_1}{\Delta'} e^{-Dj} &  \frac{y}{D}\ln R \le j < \frac{4\ln T}{\Delta'^2},\\
   \displaystyle  \frac{16}{T} & j \ge \frac{4\ln T}{\Delta'^2},
   \end{array}\right.
\end{eqnarray*}
where the outer expectation is taken over $s(j)$ distributed as $\mbox{Binomial}(j, \mu_1)$.
%Then, from the time step at which the first arm is played for the $j$th time, the expected number of steps until (and including) time $t$ such that $\theta_1(t)\ge y$, is bounded by 
%and $D$ denote the KL-divergence between $\mu_1$ and $y$, i.e.
% $$D= y \ln \frac{y}{\mu_1} + (1-y) \ln \frac{1-y}{1-\mu_1}.$$ 
 %$$ \Ex[X_{j}] \le 2 + \frac{\mu_1}{\Delta} \exp\{-Dj\} + \frac{4}{\Delta}.$$
\end{lemma} 
\begin{proof}
The complete proof of this lemma is included in Appendix \ref{app:lem:exj}; here we provide some high level ideas.

Using Lemma \ref{lem:exj-1}, the expected value of $X(j, s(j), y)$ for any given $s(j)$,
\begin{center} $\ExR{X(j,s(j),y) \condR s(j)} = \frac{1}{F^B_{j+1,y}(s(j))} - 1.$ \end{center}
For large $j$, i.e., $j \ge 4 (\ln T)/ \Delta'^2$, we use Chernoff--Hoeffding bounds to argue that with probability at least  ($1-\frac{8}{T^2}$), $s(j)$ will be 
%close its mean $\mu_1 j$, and 
greater than $\mu_1 j - \Delta' j/2$. And, for $s(j) \ge \mu_1 j - \Delta' j/2 = yj + \Delta' j /2$, we can show that the probability  $F^B_{j+1, y}(s(j))$ will be at least $1-\frac{8}{T^2}$, again using Chernoff--Hoeffding bounds. These observations allow us to derive that $\ExR{\ExB{\min\{X(j,s(j),y), T\}}} \le \frac{16}{T}$, for $j \ge 4 (\ln T)/ \Delta'^2$. 

For small $j$, the argument is more delicate. In this case, $s(j)$ could be small with a significant probability. More precisely, $s(j)$ could take a value $s$ smaller than $y j$ with binomial probability $f^B_{j,\mu_1}(s)$. For such $s$, we use the lower bound $F^B_{j+1, y}(s) \ge (1-y) F^B_{j,y}(s) + y F^B_{j,y}(s-1) \ge  (1-y) F^B_{j,y}(s) \ge (1-y) f^B_{j,y}(s)$, and then bound the ratio ${f^B_{j,\mu_1}(s)}/{f^B_{j,y}(s)}$ in terms of $\Delta'$, $R$ and KL-divergence $D$. For $s(j) = s\ge \lceil yj \rceil$, we use the observation that since $\lceil yj \rceil$ is greater than or equal to the median of $\mbox{Binomial}(j,y)$ (see \cite{JogdeoS}), we have $F^B_{j,y}(s) \ge 1/2$ . After some algebraic manipulations, we get the result of the lemma.
\end{proof}

Using Lemma~\ref{lem:E2}, and Lemma~\ref{lem:exj} for $y=\mu_2+\Delta/2$, and $\Delta'=\Delta/2$, we can bound the expected number of plays of the second arm as:
\begin{equation}
\label{eq:k2T}
\begin{array}{rcl}
\Ex[k_2(T)] & = & L  + \ExR{\sum_{j=j_0}^{T-1} Y_{j}} \vspace{0.1in}\\
& \le & L + \sum_{j=0}^{T-1} \ExR{\ExB{\min\{X(j,s(j), \mu_2+\frac{\Delta}{2}), T\} \condB s(j)}\ } + \remove{2} \add{\sum_{t=1}^{T} T\cdot \Pr(\overline{E_2(t)})} \vspace{0.1in}\\
& \le & L + \frac{4\ln T}{\Delta'^2} + \sum_{j=0}^{4(\ln T)/\Delta'^2-1} \frac{\mu_1}{\Delta'} e^{-Dj} + \left(\frac{y}{D}\ln R\right) \frac{2}{1-y} +  \sum_{j=\frac{y}{D}\ln{R}}^{4(\ln T)/\Delta'^2-1} \frac{R^y e^{-Dj}}{1-y} + \frac{16}{T} \cdot T + 2 \vspace{0.1in}\\
& \le &  \remove{\frac{32\ln T}{\Delta^2}} \add{\frac{40\ln T}{\Delta^2}} + \frac{48}{\Delta^4} + 18,
\end{array}
\end{equation}
where the last inequality is obtained after some algebraic manipulations; details are provided in Appendix \ref{app:eq:k2T}.

This gives a regret bound of 
$$\Ex[{\cal R}(T)] =\ExR{ \Delta \cdot \ k_2(T)} \le \left(\frac{\remove{32} \add{40}\ln T}{\Delta} + \frac{48}{\Delta^3} + 18\Delta\right). $$

%%%%%%%%%%%%%%%%%%%%%%%%%%%%
\section{Regret bound for the $N$-armed bandit problem}
 In this section, we prove Theorem \ref{th:N-arms}, our result for the $N$-armed bandit problem. Again, we assume that all arms have Bernoulli distribution on rewards, and that the first arm is the unique optimal arm. 

At every time step $t$, we divide the set of suboptimal arms into saturated and unsaturated arms. We say that an arm $i \ne 1$ is in the saturated set $C(t)$ at time $t$, if it has been played at least \remove{$L_i:=\frac{16\ln T}{\Delta_i^2}$} \add{$L_i:=\frac{24\ln T}{\Delta_i^2}$} times before time $t$. We bound the regret due to playing unsaturated and saturated suboptimal arms
separately. The former is easily bounded as we will see; most of the work is in bounding the latter.  For this, we bound the number of plays of saturated arms between two consecutive plays of the first arm\remove{in which $\theta_i(t)$ of some saturated arm $i$ exceeds $\theta_1(t)$}.

\begin{figure}
\begin{center}
\includegraphics[width=150mm]{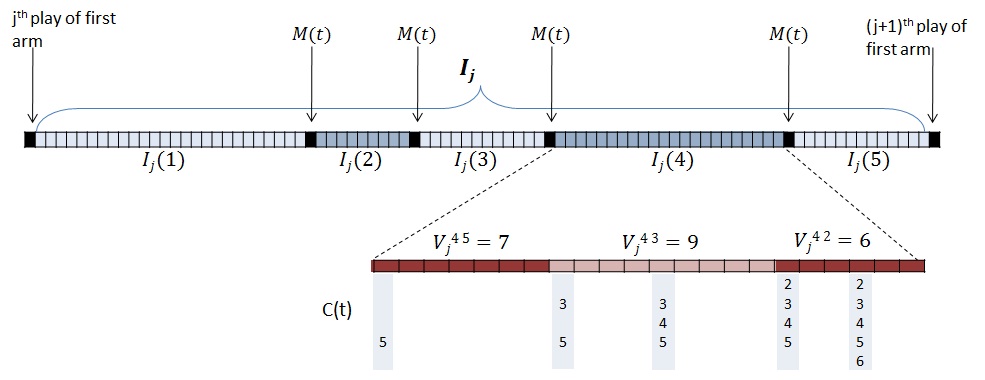}
\end{center}
\caption{\vspace{-1in} Interval $I_j$}
\label{fig:interval}
\end{figure}

In the following, by an interval of time we mean a set of contiguous time steps.
Let r.v. $I_j$ denote the interval between (and excluding) the $j^{th}$ and $(j+1)^{th}$ plays of the 
first arm.
We say that event $M(t)$ holds at time $t$, if $\theta_1(t)$ exceeds \remove{$\theta_i(t)$} \add{$\mu_i+\frac{\Delta_i}{2}$} of all the saturated arms, i.e., 
\remove{\begin{equation} 
M(t): \theta_1(t) > \max_{i\in C(t)} \theta_i(t). 
\end{equation}
}
\add{\begin{equation} 
M(t): \theta_1(t) > \max_{i\in C(t)} \mu_i + \frac{\Delta_i}{2}. 
\end{equation}
For $t$ such that $C(t)$ is empty, we define $M(t)$ to hold trivially.}

Let r.v. $\gamma_j$ denote the number of occurrences of event $M(t)$ in interval $I_j$:
\begin{equation} 
\mbox{$\gamma_j = |\{t \in I_j: M(t)=1\}|. $}
\end{equation}
Events $M(t)$ divide $I_j$ into sub-intervals in a natural way: 
For $\ell=2$ to $\gamma_j$, let r.v. $I_j(\ell)$ denote the sub-interval of $I_j$ between the $(\ell-1)^{th}$ and $\ell^{th}$ occurrences of event $M(t)$ in $I_j$ (excluding the time steps in which event $M(t)$ occurs). We also define $I_j(1)$ and $I_j(\gamma_j+1)$: If 
$\gamma_j>0$ then
$I_j(1)$ denotes the sub-interval in $I_j$ before the first occurrence of event $M(t)$ in $I_j$; and $I_j(\gamma_j+1)$ denotes the sub-interval in $I_j$ after the last occurrence of event $M(t)$ in $I_j$. For $\gamma_j=0$ we have $I_j(1) = I_j$. 

Figure \ref{fig:interval} shows an example of interval $I_j$ along with sub-intervals $I_j(\ell)$; in this figure $\gamma_j=4$.

\remove{Observe that since a saturated arm $i$ can be played at step $t$ only if $\theta_i(t)$ is greater than $\theta_1(t)$, 
saturated arms can 
only be played in the time steps belonging to intervals $I_j(\ell), \ell=1, \ldots, \gamma_j+1$ (notice,
however, that it is possible that at a step $t\in I_j(\ell)$ no saturated arm is played because an unsaturated arm has the greatest $\theta(t)$)
Then, the number of plays of saturated arms in interval $I_j$ is at most 
$$\mbox{$\sum_{\ell=1}^{\gamma_j+1} |I_j(\ell)| .$}$$}
\add{Observe that since a saturated arm $i$ can be played at step $t$ only if $\theta_i(t)$ is greater than $\theta_1(t)$, 
saturated arm $i$ can be played at a time step $t\notin I_j(\ell), \forall \ell, j$ (i.e., at a time step $t$ where $M(t)$ holds) only if $\theta_i(t) > \mu_i + \Delta_i/2$ . 
Let us define event $E(t)$ as
$$E(t): \ \ \{\theta_i(t) \in [\mu_i - \Delta_i/2, \mu_i + \Delta_i/2], \forall i\in C(t)\}.$$
Then, the number of plays of saturated arms in interval $I_j$ is at most 
\add{$$\mbox{$\sum_{\ell=1}^{\gamma_j+1} |I_j(\ell)| + \sum_{t\in I_j} I(\overline{E(t)}).$}$$}
 In words, $E(t)$ denotes the event that all saturated arms have $\theta_i(t)$ tightly concentrated around their means. Intuitively, from the definition of saturated arms, $E(t)$ should hold with high probability; we prove this in Lemma~\ref{lem:multiple-E(T)}}.

We are interested in bounding regret due to playing saturated arms, which depends not only on the number of plays,  but also on \emph{which} saturated arm is played at each time step. Let $V^{\ell, a}_j$ denote the number of steps in $I_j(\ell)$, for which $a$ is the best saturated arm, i.e. \vspace{-0.05in}
\begin{equation}
\mbox{$V^{\ell, a}_j = |\{t\in I_j(\ell) : \mu_a = \max_{i\in C(t)} \mu_i\}|, $}\vspace{-0.05in}
\end{equation}
(resolve the ties for best saturated arm using an arbitrary, but fixed, ordering on arms).
\remove{Note that $|I_j(\ell)| = \sum_{a=2}^N V^{\ell,a}_j.$}
%Then, the number of plays of saturated arms in $I_j$ is at most
%$$\sum_{\ell=1}^{\gamma_j+1} \sum_{a=2}^{N} V^{\ell, a}_j.$$
In Figure \ref{fig:interval}, we illustrate this notation by showing steps $\{V^{4,a}_j\}$ for interval $I_j(4)$. In the example shown, we assume that $\mu_1 > \mu_2 > \cdots > \mu_6$, and that the suboptimal arms got added to the saturated set $C(t)$ in order $5, 3 , 4, 2, 6$, so that initially $5$ is the best saturated arm, then $3$ is the best saturated arm, and finally $2$ is the best saturated arm. 

\add{Recall that $M(t)$ holds trivially for all $t$ such that $C(t)$ is empty. Therefore, there is at least one saturated arm at all $t\in I_j(\ell)$, and hence $V^{\ell,a}_j, a=2,\ldots, N$ are well defined and cover the interval $I_j(\ell)$,
 $$\mbox{$|I_j(\ell)| = \sum_{a=2}^N V^{\ell,a}_j.$}$$
 }
\remove{Next, we will show that, with high probability, the regret due to playing a saturated arm during one of the $V^{\ell, a}_j$ steps is at most $3\Delta_a$. The idea is that since saturated arms have their $\theta_i(t)$s tightly concentrated around their means $\mu_i$, so with high probability, either the arm with the highest mean (i.e., the best saturated arm $a$) or an arm with  mean very close to $\mu_a$ will be chosen to be played during these $V^{\ell,a}_j$ steps.
More precisely, define $E(T)$ to be the event that for all $t \in [1,T]$, and for all $i\in C(t)$, $\mu_i -\Delta_i/2 \le \theta_i(t) \le \mu_i + \Delta_i/2$. Then, simple application of Chernoff-Hoeffding bounds, we have
\begin{lemma} \label{lem:multiple-E(T)}
$\Pr(E(T)) \ge 1-\frac{4(N-1)}{T}.$ 
\end{lemma}
\begin{proof}
Refer to Appendix \ref{app:lem:multiple-E(T)} for details.	
\end{proof}}
\add{Next, we will show that the regret due to playing a saturated arm at a time step $t$ in one of the $V^{\ell, a}_j$ steps is at most $3\Delta_a + I(\overline{E(t)})$. The idea is that if all saturated arms have their $\theta_i(t)$ tightly concentrated around their means $\mu_i$, then either the arm with the highest mean (i.e., the best saturated arm $a$) or an arm with mean very close to $\mu_a$ will be chosen to be played during these $V^{\ell,a}_j$ steps. That is,}
\remove{Given that $E(T)$ holds, }
if a saturated arm $i$ is played at a time $t$ among one of the $V^{\ell,a}_j$ steps, 
then, \add{either $E(t)$ is violated, i.e. $\theta_{i'}(t)$ for some saturated arm $i'$ is not close to its mean, or} \vspace{-0.1in}
$$\mu_i + \Delta_i / 2 \ge \theta_i(t) \ge \theta_a(t) \ge \mu_a - \Delta_a/2, \vspace{-0.1in}$$
which implies that \vspace{-0.1in}
\begin{equation} 
\mbox{$\Delta_i = \mu_1 - \mu_i \le \mu_1 - \mu_a + \frac{\Delta_a}{2} + \frac{\Delta_i}{2} \Rightarrow \Delta_i \le 3\Delta_a$}.\vspace{-0.05in}
\end{equation}

\add{
Therefore, regret due to play of a saturated arm at a time $t$ in one of the $V^{\ell,a}_j$ steps is at most $3\Delta_a + I(\overline{E(t)})$. With slight abuse of notation let us use $t\in V^{\ell,a}_j$ to indicate that $t$ is one of the $V^{\ell,a}_j$ steps in $I_j(\ell)$. Then, the expected regret \emph{due to playing saturated arms} in interval $I_j$ is bounded as
\begin{eqnarray} 
\label{eq:R-V}
\mbox{$\ExR{{\cal R}^s(I_j)}$} &\le & \mbox{$\ExR{\sum_{\ell=1}^{\gamma_j+1} \sum_{a=2}^N  \sum_{t\in V^{\ell,a}_j} (3\Delta_a + I(\overline{E(t)}))} + \sum_{t\in I_j} I(\overline{E(t)}).$}\nonumber\\
 &= & \mbox{$\ExR{\sum_{\ell=1}^{\gamma_j+1} \sum_{a=2}^N 3\Delta_a V^{\ell,a}_j} + 2 \ExR{\sum_{t\in I_j} I(\overline{E(t)})}.$}
\end{eqnarray}
}
\remove{
Therefore, given event $E(T)$ regret due to play of a saturated arm in one of the $V^{\ell,a}_j$ steps is at most $3\Delta_a$, and thus, using Lemma~\ref{lem:multiple-E(T)},the expected regret \emph{due to playing saturated arms} in interval $I_j$ is bounded as,
\begin{equation} 
\label{eq:R-V}
\mbox{$\ExR{{\cal R}^s(I_j)} \le \ExR{\sum_{\ell=1}^{\gamma_j+1} \sum_{a=2}^N  (3\Delta_a) V^{\ell, a}_j \condR E(T)} + \frac{4N|I_j|}{T}.$}
\end{equation}
}
\add{The following lemma will be useful for bounding the second term on the right hand side in the above equation (as shown in the complete proof in Appendix~\ref{app:N-arms-details}).
%We prove the following lemma to bound $\Pr(\overline{E(t)})$.
\begin{lemma} \label{lem:multiple-E(T)}
For all $t$,
$$\Pr(E(t))\ge 1-\frac{4(N-1)}{T^2}.$$
Also, for all $t, j$, and $s\le j$,
$$\Pr(E(t) \ |\  s(j)=s) \ge 1-\frac{4(N-1)}{T^2}.$$
\end{lemma}
\begin{proof}
Refer to Appendix \ref{app:lem:multiple-E(T)}.	
\end{proof}}
The stronger bound given by the second statement of lemma above will be useful later in bounding the first term on the rhs of \eqref{eq:R-V}.
%Recall that $V^{\ell,a}_j$ denotes the number of steps before either event $M(t): \theta_1(t) \ge \max_i \theta_i(t)$ happens or some arm other than $a$ becomes the best saturated arm. Also, recall that $X(j, s, y)$ was defined as the geometric variable denoting the number of trials until an independent sample from $\mbox{Beta}(s+1,j-s+1)$ distribution exceeds $y$. 
For bounding that term, we establish the following lemma.
%Now, given event $E(T)$, while $a$ is the best saturated arm,  the event $M'(t)$, defined below, implies event $M(t)$, 
%\begin{equation}
%M'(t):\theta_1(t) > \mu_a+\Delta_a/2.
%\end{equation}
%More precisely, if $a$ is the best saturated arm, then $\mu_a+\Delta_a/2 =\max_{i\in C(t)}  \mu_i+\Delta_i/2$, and if $E(T)$ holds, then $ \max_{i\in C(t)}  \mu_i+\Delta_i/2 \ge \max_{i\in C(t)} \theta_i(t)$, therefore $M'(t) \Rightarrow \theta_1(t) > \max_{i\in C(t)} \theta_i(t) \equiv M(t)$.
%This means that given event $E(T)$, and given  $s(j)=s$, $V^{\ell, a}_j$ is stochastically dominated by geometric random variable $X(j,s, \mu_a+\Delta_a/2)$ (which was defined as the number of trials until an independent sample from $\mbox{Beta}(s+1,j-s+1)$ distribution exceeds $\mu_a+\Delta_a/2$).
\begin{lemma}
\label{lem:GammaV-dep} 
For all $j$,
%$$\ExR{\sum_{\ell=1}^{\gamma_j+1} |I_j^{\ell}|	\ \ \condR s(j), E(T)} \le  \ExR{\sum_{\ell=1}^{\gamma_j+1} \sum_{a=1}^N\ExB{X(j,s(j),\mu_a + \frac{\Delta_a}{2})\condB s(j)} \condR s(j), E(T)}$$
%Also, similarly
\remove{
\begin{equation}\vspace{-0.15in}
\label{eq:GammaV-dep}
\mbox{$\ExR{\sum_{\ell=1}^{\gamma_j+1} \sum_a V_j^{\ell,a} \Delta_a|	\ \ E(T)} \le  \ExR{ (\gamma_j+1) \sum_{a=2}^N  \Delta_a \ExB{\min\{X(j,s(j),\mu_a + \frac{\Delta_a}{2}), T\}\condB s(j)} \condR E(T)} $}
\end{equation}
}
\add{
\begin{equation}\vspace{-0.15in}
\label{eq:GammaV-dep}
\mbox{$
\ExR{\sum_{\ell=1}^{\gamma_j+1} \sum_a V_j^{\ell,a} \Delta_a} \le  \ExR{ \ExB{(\gamma_j+1) \condB s(j) }\sum_{a=2}^N  \Delta_a \ExB{\min\{X(j,s(j),\mu_a + \frac{\Delta_a}{2}), T\}\condB s(j)}} 
$}
\end{equation}
}
\end{lemma}
\begin{proof}
The key observation used in proving this lemma is that given \remove{event $E(T)$ and } a fixed value of $s(j)=s$, the random variable $V^{\ell, a}_j$ is stochastically dominated by random variable $X(j, s, \mu_a+\frac{\Delta_a}{2})$ (defined earlier as a geometric variable denoting the number of trials before an independent sample from $\mbox{Beta}(s+1,j-s+1)$ distribution exceeds $\mu_a+\frac{\Delta_a}{2}$). A technical difficulty in deriving the inequality above is that the random variables $\gamma_j$ and $V^{\ell,a}_j$ are not independent in general (both depend on the values taken by $\{\theta_i(t)\}$ over the interval). This issue is handled through careful conditioning of the random variables on history. The details of the proof are provided in Appendix \ref{app:GammaV-dep}.
\end{proof}
%Using previous lemma  
%\begin{eqnarray}
%\label{eq:tmp1}
%\Ex[{\cal R}^s(I_j)|E(T)] & \le & \Ex[\sum_{\ell=1}^{\gamma_j+1} \sum_{a=2}^N (3 \Delta_a) V^{\ell, a}_j | E(T)] \nonumber\\
%& \le & \Ex[(\gamma_j +1) \sum_{a=2}^N  (3 \Delta_a) \min\{ \Ex[X(j,s(j), \mu_a + \frac{\Delta_a}{2})|s(j)], T \}\:\:\: 
%\end{eqnarray}

%Using previous lemma and Equation \eqref{eq:R-V}, the total expected regret due to saturated arms, given $E(T)$, is bounded as,
%\begin{eqnarray}
%\label{eq:tmp1-0}
%\Ex[{\cal R}^s(T)|E(T)] & = & \sum_{j=0}^{T-1} 
%\Ex[{\cal R}^s(I_j)|E(T)] 
%& \le  & \sum_{j=0}^{T-1} \Ex[\gamma_j \sum_a  (3 \Delta_a) \Ex[X(j,s(j), \mu_a + \frac{\Delta_a}{2})|s(j)]] \nonumber\\
%& & \hspace{0.5in}+ \sum_{j=0}^{T-1} \Ex[\sum_a  (3 \Delta_a) \min\{ \Ex[X(j,s(j), \mu_a + \frac{\Delta_a}{2})|s(j)], T \}]  \nonumber\\
%& \le & \sum_{j=0}^{T-1}  \Ex\left[\gamma_j  \sum_a (3 \Delta_a) \left(\frac{1}{F_{j+1, y_a}(s(j))}-1\right)\right] \nonumber\\
%& & \hspace{0.5in} +   \sum_j \Ex[\sum_a  (3 \Delta_a) \min\{ \Ex[X(j,s(j), \mu_a + \frac{\Delta_a}{2})|s(j)], T\} ],
%\end{eqnarray}
%where $y_a=\mu_a+\Delta_a/2$. The last equality follows from Lemma \ref{lem:exj-1}, which gives the expression for the expected value of random variable $X(j,s,y)$. In this section, we abbreviate the symbol $F^B_{n,p}$ for the cdf of $\mbox{Binomial}(n,p)$ distribution to $F_{n,p}$.

Now using the above lemma the first term in \eqref{eq:R-V} can be bounded by
\begin{equation} 
\mbox{$
3 \ExR{ \sum_{j=0}^{T-1} \ExB{\gamma_j \condB s(j)} \sum_a \Delta_a \ExB{\min\{X(j,s(j), y_a), T\}\condB s(j)}} + 
3 \ExR{ \sum_a \Delta_a \ExB{\min\{X(j,s(j), y_a), T\}\condB s(j)}}. $}\nonumber
\end{equation}

We next show how to bound the first term in this equation; the second term will be dealt with in the complete proof in Appendix~\ref{app:N-arms-details}.

Recall that $\gamma_j$ denotes the number of occurrences of event $M(t)$ in interval $I_j$, i.e. the number of times in interval $I_j$,  $\theta_1(t)$ was greater than \remove{$\theta_i(t)$} \add{$\mu_i+\frac{\Delta_i}{2}$} of all saturated arms $i\in C(t)$, and yet the first arm was not played. The only reasons the first arm would not be played at a time $t$ despite of $\theta_1(t) > \max_{i\in C(t)} \mu_i+\frac{\Delta_i}{2}$  \add{are that either $E(t)$ was violated, i.e. some saturated arm whose $\theta_i(t)$ was not close to its mean was played instead; or} some unsaturated arm
$u$ with highest $\theta_u(t)$ was played. \remove{And, since an unsaturated arm $u$ can be played for at most $L_u$ times before it becomes saturated,}
\add{Therefore, the random variables $\gamma_j$ satisfy 
\begin{equation}
\mbox{$\gamma_j \le \sum_{t\in I_j} I(\mbox{an unstaurated arm is played at time $t$}) + \sum_{t\in I_j}I(\overline{E(t)}).$} \nonumber
\end{equation}
Using Lemma \ref{lem:multiple-E(T)}, and the fact that an unsaturated arm $u$ can be played at most $L_u$ times before it becomes saturated, we obtain that 
\begin{eqnarray}
\label{eq:ub-gamma}
\mbox{$\sum_{j=0}^{T-1} \Ex[\gamma_j|s(j)]$} & \le & \mbox{$\Ex[\sum_{t=1}^T I(\mbox{an unstaurated arm is played at time $t$})|s(j)]+ \sum_{j=0}^{T-1} \Ex[\sum_{t\in I_j} I(\overline{E(t)}) | s(j)]$}\nonumber\\
& \le & \mbox{$\sum_u L_u + \sum_{j=0}^{T-1} \sum_{t=1}^T \Pr(\overline{E(t)} | s(j))$}\nonumber\\
&  \le & \mbox{$\sum_u L_u +4(N-1).$}
\end{eqnarray}
}

Note that $\sum_{j=0}^{T-1} \Ex[\gamma_j|s(j)]$ is a r.v. (because of random $s(j)$), and the above bound applies
for all instantiations of this r.v.

Let $y_a = \mu_a+\frac{\Delta_a}{2}$. Then,
\remove{Since the random variables $\ExR{X(j,s(j),y_a)\condR s(j)}$ and $\gamma_j$ are not independent (both of them depend on $s(j)$), we are not able to use the implied bound on $\ExR{\sum_j \gamma_j}$ directly to bound $\sum_j \ExR{\gamma_j \ExB{X(j,s(j), y_a)\condB s(j)}}$. Instead, we use the above bound to argue that the worst case is the one with $\gamma_{j^*_a} = \sum_u L_u$, and $\gamma_j=0$, for all $j\ne j^*_a$, where
 \vspace{-0.1in}
$$j^*_a=\arg \max_{j \in \{0, \ldots, T-1\}} \ExB{X(j,s(j), y_a)\condB s(j)} = \arg \max_{j \in \{0, \ldots, T-1\}} \frac{1}{F_{j+1, y_a}(s(j))}. $$
Note that $j^*_a$ is a random variable, which is completely determined by the instantiation of random sequence $s(1), s(2), \ldots$. More precisely,
}
%Then, using above worst case upper bound on sum of $\gamma_j$'s, and Lemma \ref{lem:exj-1} for the expected values of random variables $X(j,s(j),y_a)$, we obtain, \vspace{-0.1in}
%for the first term on the right hand side of \eqref{eq:tmp1-0} we have
\begin{eqnarray}
\label{eq:tmp3}
& & \hspace{-0.25in} \mbox{$ \ExR{ \sum_{j=0}^{T-1} \ExB{\gamma_j \condB s(j)} \sum_a \Delta_a \ExB{X(j,s(j), y_a)\condB s(j)} \remove{\condR E(T)}}$} \nonumber\\
& \le & \mbox{$ \ExR{ \left(\sum_{j=0}^{T-1} \ExB{\gamma_j \condB s(j)} \right) \left( \max_j \sum_a \Delta_a \ExB{X(j,s(j), y_a)\condB s(j)}\right) \remove{\condR E(T)}}$} \nonumber\\
%& = & \sum_{j=0}^{T-1} \Ex[\gamma_j \sum_a \frac{3\Delta_a}{F_{j+1, y_a}(s(j))} | E(T)] \nonumber\\
%& = & \sum_a \sum_{j=0}^{T-1}  \Ex[\gamma_{j} \frac{\Delta_a}{F_{j+1, y_a}(s(j))} | E(T)] \nonumber\\
& \le & \mbox{$(\sum_u L_u \add{+4(N-1)}) \sum_a  \Delta_a \ExR{\max_j \ExB{X(j,s(j), y_a)\condB s(j)}}$} \nonumber\\
& \le & \mbox{$(\sum_u L_u \add{+4(N-1)}) \sum_a   \Delta_a \ExR{\frac{\Delta_a}{F_{j^*_a+1, y_a}(s(j^*_a))}\cdot I(s(j^*_a) \le \lfloor y_aj^*_a\rfloor) + \frac{\Delta_a}{F_{j^*_a+1, y_a}(s(j^*_a))}\cdot I(s(j^*_a) \ge \lceil y_aj^*_a\rceil)}$}, \vspace{-0.1in}\nonumber\\
%& \le &  (\sum_u L_u)  \sum_a  \Ex[\frac{3\Delta_a}{F_{j^*_a+1, y_a}(s(j^*_a))}\cdot I(s(j^*_a) \le \lfloor y_aj^*_a\rfloor)] + (\sum_u L_u) \cdot \frac{6\Delta_a}{1-y_a}, 
\end{eqnarray}
\add{where
 \vspace{-0.1in}
$$j^*_a=\arg \max_{j \in \{0, \ldots, T-1\}} \ExB{X(j,s(j), y_a)\condB s(j)} = \arg \max_{j \in \{0, \ldots, T-1\}} \frac{1}{F_{j+1, y_a}(s(j))}. $$
Note that $j^*_a$ is a random variable, which is completely determined by the instantiation of random sequence $s(1), s(2), \ldots$.}
 Now, for the first term in above, 
\begin{eqnarray}
\mbox{$\ExR{\frac{1}{F_{j^*_a+1, y_a}(s(j^*_a))} \cdot I(s(j^*_a) \le \lfloor y_aj^*_a\rfloor)} $}& \le & \mbox{$\sum_j \ExR{\frac{1}{F_{j+1, y_a}(s(j))}\cdot I(s(j) \le \lfloor y_aj\rfloor)}$} \nonumber\\
& = & \mbox{$\sum_j \sum_{s=0}^{\lfloor y_aj\rfloor} \frac{f_{j,\mu_1}(s)}{F_{j+1, y_a}(s)}$} \nonumber\\
& \le & \mbox{$\sum_j \frac{\mu_1}{\Delta'_a} e^{-D_a j} \le \frac{16}{\Delta_a^3}$},\nonumber
\end{eqnarray}
where $\Delta'_a = \mu_1-y_a = \Delta_a/2$, $D_a$ is the KL-divergence between Bernoulli distributions with parameters $\mu_1$ and $y_a$. The penultimate inequality follows using \eqref{eq:smalls-bound} in the proof of Lemma \ref{lem:exj} \add{in Appendix \ref{app:lem:exj}}, with $\Delta'=\Delta'_a$, and $D = D_a$. The resulting bound on the first term in \eqref{eq:tmp3} is $O((\sum_u L_u)\sum_a \frac{\Delta_a}{\Delta_a^3}) = O(\left(\sum_a \frac{1}{\Delta_a^2}\right)^2 \ln T)$, which forms the dominating term in our regret bound. The bounds on the remaining terms and further details of the proof for regret due to saturated arms are provided in Appendix~\ref{app:N-arms-details}. 
Since an unsaturated arm $u$ becomes saturated after $L_u$ plays, regret due to unsaturated arms is at most $\sum_{u=2}^N L_u \Delta_u  = \remove{16}\add{24} (\ln T) \left(\sum_{u=2}^N \frac{1}{\Delta_u}\right)$. Summing the regret due to saturated and unsaturated arms, we obtain the result of Theorem~\ref{th:N-arms}.
\paragraph{Conclusion.}
In this paper, we showed theoretical guarantees for Thompson Sampling close to other state of the art methods, like UCB. 
Our result is a first step in theoretical understanding of TS and there are several avenues to explore for
the future work: There is a gap between our upper bounds and the lower bound of \cite{LaiR}.  While it may be easy to improve the constant factors in our upper bounds by making the analysis more careful (but more complicated), it seems harder to improve the dependence on the $\Delta$'s. With further work, we hope that our techniques in this paper will be useful in providing several extensions, including analysis of TS for delayed and batched feedbacks, contextual bandits, prior mismatch and posterior reshaping discussed in \cite{CLi}. As mentioned before, empirically TS has been shown to have superior performance than other methods, especially for handling delayed feedback. 
A theoretical justification of this observation would require a tighter analysis of TS than what we have achieved here, and in addition, it would require lower bound on the regret of the other algorithms. 
TS has also been used for problems such as regularized logistic regression (see \cite{CLi}). 
These multi-parameter settings lack theoretical analysis.

% Acknowledgments---Will not appear in anonymized version
%\acks{} 

\bibliography{Thompson_bib}
\bibliographystyle{abbrv} 

\appendix

\section{Multiple optimal arms}
\label{app:mul-opt}
Consider the $N$-armed bandit problem with $\mu^*=\max_i \mu_i$. We will show that adding another arm with expected reward $\mu^*$ can only decrease the expected regret of TS algorithm. Suppose that we added arm $N+1$ with expected reward $\mu^*$. Consider the expected regret for the new bandit in time $T$, conditioned on the exact time steps among $1, \ldots, T$, on which arm $N+1$ is played by the algorithm. Since the arm $N+1$ has expected reward $\mu^*$, there is no regret in these time steps. Now observe that in the remaining time steps, the algorithm behaves exactly as it would for the original bandit with $N$ arms. Therefore, given that the $(N+1)^{th}$ arm is played $x$ times, the expected regret in time $T$ for the new bandit will be same as the expected regret in time $T-x$ for the original bandit. Let ${\cal R}^{N}(T)$ and ${\cal R}^{N+1}(T)$ denote the expected regret in time $T$ for the original and new bandit, respectively. Then,
{\small
\begin{eqnarray*}
\ExR{{\cal R}^{N+1}(T)} = \ExR{\ExB{{\cal R}^{N+1}(T) \condB k_{N+1}(T)}} & = & \ExR{\ExB{{\cal R}^{N}(T-k_{N+1}(T)) \condB k_{N+1}(T)}}\\
& \le & \ExR{\ExB{{\cal R}^{N}(T) \condB k_{N+1}(T)}} = \ExR{{\cal R}^N(T)}.
\end{eqnarray*}
}
This argument shows that the expected regret of Thompson Sampling for the $N$-armed bandit problem with $r$ optimal arms is bounded by the expected regret of Thompson Sampling for the $(N-r+1)$-armed bandit problem obtained on removing (any) $r-1$ of the optimal arms.

\section{Facts used in the analysis}
\label{app:facts}
\begin{fact}
\label{fact:beta-binomial}
%\begin{center}
$$ F^{beta}_{\alpha, \beta}(y) = 1-F^B_{\alpha+\beta-1,y}(\alpha-1),$$
%\end{center}
for all positive integers $\alpha, \beta$.
\end{fact}
\begin{proof}
This fact is well-known (it's mentioned on Wikipedia) but we are not aware of a specific reference.
Since the proof is easy and short we will present a proof here.  
The Wikipedia page also mentions that it can be proved using integration by parts. Here we provide
a direct combinatorial proof which may be new.

One well-known way to generate a r.v. with cdf $F^{beta}_{\alpha, \beta}$ for integer $\alpha$ and $\beta$ 
is the following: generate uniform
in $[0,1]$ r.v.s $X_1, X_2, \ldots, X_{\alpha+\beta-1}$ independently. Let the values of these r.v. in 
sorted increasing order be denoted $X_1^{\uparrow}, X_2^{\uparrow}, \ldots, X_{\alpha+\beta-1}^{\uparrow}$. Then 
$X_{\alpha}^{\uparrow}$ has cdf $F^{beta}_{\alpha, \beta}$. 
Thus $F^{beta}_{\alpha, \beta}(y)$ is the probability that $X_{\alpha}^{\uparrow} \leq y$.

We now reinterpret this probability using the binomial distribution: The event 
$X_{\alpha}^{\uparrow} \leq y$ happens iff for at least $\alpha$ of the $X_1, \ldots, X_{\alpha+\beta-1}$  we
have $X_i \leq y$. 
For each $X_i$ we have $\Pr[X_i \le y] = y$; thus the probability that for at most $\alpha-1$ of the 
$X_i$'s we have $X_i \leq y$ is $F^{B}_{\alpha+\beta-1, y}(\alpha-1)$.  And so the probability that for
at least $\alpha$ 
of the $X_i$'s we have $X_i \leq y$ is $1-F^{B}_{\alpha+\beta-1, y}(\alpha-1)$.
\end{proof}

The median of an integer-valued random variable $X$ is an integer $m$ such that 
$\Pr(X \leq m) \geq 1/2$ and $\Pr(X \geq m) \geq 1/2$. The following fact says that the median of the binomial distribution is close to its mean.

\begin{fact}[\cite{JogdeoS}] \label{fact:median}
Median of the binomial distribution $\text{Binomial}(n, p)$ is either $\lfloor np \rfloor$
or $\lceil np \rceil$. 
\end{fact}

\begin{fact}[(Chernoff--Hoeffding bounds)]
\label{lem:Chernoff}
Let $X_1, . . . , X_n$ be random variables with common range $[0, 1]$ and such that $\ExR{X_t \condR X_1, . . . , X_{t-1}} = \mu$. Let $S_n = X_1 + \ldots + X_n$. Then for
all $a\ge 0$,
$$\Pr(S_n \ge n\mu + a) \le e^{-2a^2/n},$$
$$\Pr(S_n \le n\mu - a) \le e^{-2a^2/n}.$$
\end{fact}

\begin{lemma}
\label{lem:Binomial}
For all $n, p\in[0,1], \delta\ge 0$,
\begin{equation}
F^B_{n,p}(np-n\delta) \le e^{-2n\delta^2}, \ \ 1-F^B_{n,p}(np+ n \delta) \le e^{-2n\delta^2},
\end{equation}
\begin{equation}
1-F^B_{n+1, p}(np + n \delta ) \le \frac{e^{4\delta}}{e^{2n\delta^2}}.
\end{equation}
\end{lemma}
\begin{proof}
The first result is a simple application of Chernoff--Hoeffding bounds from Fact~\ref{lem:Chernoff}.
For the second result, we observe that,
$$F^B_{n+1, p}(np + n \delta) = (1-p) F^B_{n,p}(np+n \delta) + p F^B_{n,p}(np + n\delta-1) \ge  F^B_{n,p}(np + n \delta-1).$$
By Chernoff--Hoeffding bounds,
$$1-F^B_{n,p}(np +\delta n-1) \le e^{-2(\delta n-1)^2/n} = e^{-2(n^2 \delta^2 + 1 -2\delta n)/n} \le e^{-2n\delta^2 + 4\delta} = \frac{e^{4\delta}}{e^{2n\delta^2}}.$$
\end{proof}

\section{Proofs of Lemmas}
%%%%%%%%%%%%%%%%%%%%%%%%%%%%%%%%%%%%%%%%%%%%%%%%%%%%%%%%%%%%%%%%%%%%%%%%%%%%%%%
\subsection{Proof of Lemma \ref{lem:E2}}
\label{app:lem:E2}
\begin{proof}
\add{
In this lemma, we lower bound the probability of $E_2(t)$ by $1-\frac{2}{T^2}$. Recall that event $E_2(t)$ holds if the following is true:
$$\{\theta_2(t) \le \mu_2 + \frac{\Delta}{2}\} \mbox{ or } \{k_2(t) < L\}.$$ 
%We need to define $A(t)$ more precisely. 
Also define $A(t)$ as the event 
$$A(t): \frac{S_2(t)}{k_2(t)} \le \mu_2 + \frac{\Delta}{4},$$
where $S_2(t), k_2(t)$ denote the number of successes and number of plays respectively of the second arm until time $t-1$.
We will upper bound the probability of $\Pr(\overline{E_2(t)}) = 1-\Pr(E_2(t))$ as:
\begin{eqnarray}
\label{eq:E2}
\Pr(\overline{E_2(t)}) & = & \Pr(\theta_2(t) \ge \mu_2 + \frac{\Delta}{2}, k_2(t) \ge L) \nonumber \\
& \le & \Pr(\overline{A(t)}, k_2(t) \ge L) + \Pr(\theta_2(t) \ge \mu_2+\frac{\Delta}{2}, k_2(t) \ge L, A(t)).
\end{eqnarray}
%where $L=32\ln T/\Delta^2$ (note that earlier, $L$ was $16\ln T/\Delta^2$, this increase in $L$ will case a small increase in the constants in the result.).
For clarity of exposition, let us define another random variable $\overline{Z}_{2,M}$, as the average number of successes over the first $M$ plays of the second arm. More precisely, let random variable $Z_{2,m}$ denote the output of the $m^{th}$ play of the second arm.  Then,
$$\overline{Z}_{2,M} = \frac{1}{M} \sum_{m=1}^{M} Z_{2,m}.$$
Note that by definition, $\overline{Z}_{2,k_2(t)} = \frac{S_2(t)}{k_2(t)}$. Also, $\overline{Z}_{2,M}$ is the average of $M$ iid Bernoulli variables, each with mean $\mu_2$. \newline
%Now, we want to prove that for all time $t$, $\Pr(\overline{A(t)}, k_2(t) \ge L)$ is small (this was not stated clearly earlier).
%$\Pr(\frac{Z_{2,k_2(t)}}{k_2(t)} \ge \mu_2 + \frac{\Delta}{4}, k_2(t) = \ell)$ is small.
Now, for all $t$,
$$
\begin{array}{rcl}
\Pr(\overline{A(t)}, k_2(t) \ge L) & = & \sum_{\ell =L}^{T} \Pr(\overline{Z}_{2,k_2(t)} \ge  \mu_2 + \frac{\Delta}{4}, k_2(t) = \ell) \\
& = & \sum_{\ell=L}^T\Pr(\overline{Z}_{2,\ell} \ge \mu_2 + \frac{\Delta}{4}, k_2(t) = \ell) \\
& \le & \sum_{\ell=L}^T\Pr(\overline{Z}_{2,\ell} \ge \mu_2 + \frac{\Delta}{4}) \\
& \le & \sum_{\ell=L}^T e^{-2\ell\Delta^2/16} \\
& \le & \frac{1}{T^2}.
\end{array}
$$
The second last inequality is by applying Chernoff bounds, since $\overline{Z}_{2,\ell}$ is simply the average of $\ell$ iid Bernoulli variables each with mean $\mu_2$. \newline\\
We will derive the bound on second probability term in \eqref{eq:E2} in a similar manner. It will be useful to define $W(\ell, z)$ as a random variable distributed as $\mbox{Beta}(\ell z+1, \ell-\ell z+1)$. Note that if at time $t$, the number of plays of second arm is $k_2(t)=\ell$, then $\theta_2(t)$ is distributed as $\mbox{Beta}(\ell \overline{Z}_{2,\ell}+1, \ell-\ell \overline{Z}_{2,\ell}+1)$, i.e. same as $W(\ell, \overline{Z}_{2,\ell})$.
\begin{eqnarray*}
\Pr(\theta_2(t) > \mu_2 + \frac{\Delta}{2}, A(t), k_2(t) \ge L)& = & \sum_{\ell=L}^T \Pr(\theta_2(t) > \mu_2 + \frac{\Delta}{2}, A(t), k_2(t) = \ell)\\
& \le & \sum_{\ell=L}^T \Pr(\theta_2(t) > \frac{S_2(t)}{k_2(t)} - \frac{\Delta}{4} + \frac{\Delta}{2}, k_2(t) =\ell)  \\
& = & \sum_{\ell=L}^T \Pr(W(\ell, \overline{Z}_{2,\ell}) > \overline{Z}_{2,\ell} + \frac{\Delta}{4}, k_2(t) =\ell) \\
& \le & \sum_{\ell=L}^T \Pr(W(\ell, \overline{Z}_{2,\ell}) > \overline{Z}_{2,\ell} + \frac{\Delta}{4}) \\
\mbox{(using Fact 1) } & = & \sum_{\ell=L}^T \ExR{F^B_{\ell+1, \overline{Z}_{2,\ell} +\frac{\Delta}{4}}(\ell\overline{Z}_{2,\ell})}  \\
& \le & \sum_{\ell=L}^T \ExR{F^B_{\ell, \overline{Z}_{2,\ell}+\frac{\Delta}{4}}(\ell\overline{Z}_{2,\ell})}\\
& \le & \sum_{\ell=L}^T \exp\{-\frac{2\Delta^2\ell^2/16}{\ell}\} \\
& \le & T e^{-2L\Delta^2/16} = \frac{1}{T^2}.
\end{eqnarray*}
%Here, we used the observation that when $k_2(t)=\ell$, $\theta_2(t)$ is distributed as $\mbox{Beta}(\ell \overline{Z}_{2,\ell}+1, \ell-\ell \overline{Z}_{2,\ell}+1)$, and then used Fact \ref{fact:beta-binomial} for relating cdf of Beta distribution to cdf of Binomial distribution.
The third-last inequality follows from the observation that 
$$F^B_{n+1, p}(r) = (1-p)F^B_{n,p}(r) + p F^B_{n,p}(r-1) \le (1-p)F^B_{n,p}(r) + p F^B_{n,p}(r) = F^B_{n,p}(r). $$
And, the second-last inequality follows from Chernoff--Hoeffding bounds 
(refer to Fact \ref{lem:Chernoff} and Lemma \ref{lem:Binomial}).
}
\remove{Summing above over $t=1, \ldots, T$, we get the result of the lemma.}
\end{proof}

%%%%%%%%%%%%%%%%%%%%%%%%%%%%%%%%%%%%%%%%%%%%%%%%%%%%%%%%%%%%%%%%%%%%%%%%%%%%%%%
\subsection{Proof of Lemma~\ref{lem:exj}}
\label{app:lem:exj}
\begin{proof}
Using Lemma \ref{lem:exj-1}, the expected value of $X(j, s(j), y)$ for any given $s(j)$,
$$\ExR{X(j,s(j),y) \condR s(j)} = \frac{1}{F^B_{j+1,y}(s(j))} - 1.$$ 

\paragraph{Case of large $j$:}
First, we consider the case of large $j$, i.e. when $j \ge 4(\ln T)/\Delta'^2$. Then, by simple application of Chernoff--Hoeffding bounds (refer to Fact \ref{lem:Chernoff} and Lemma \ref{lem:Binomial}), we can derive that for any $s \ge (y + \frac{\Delta'}{2})j$,
$$F^B_{j+1, y}(s) \ge F^B_{j+1, y}(yj + \frac{\Delta' j}{2}) \ge 1-\frac{e^{4\Delta'/2}}{e^{2j\Delta'^2/4}} \ge 1-\frac{e^{2\Delta'}}{T^2} \ge 1-\frac{8}{T^2},$$
giving that for $s \ge y(j + \frac{\Delta'}{2})$, $\ExR{X(j+1, s, y)} \le \frac{1}{(1-\frac{8}{T^2})}-1$.

Again using Chernoff--Hoeffding bounds, the probability that $s(j)$ takes values smaller than $(y + \frac{\Delta'}{2})j$ 
%= (\mu_1 - \frac{\Delta }{2})j$ 
can be bounded as,
$$F^B_{j,\mu_1}(yj + \frac{\Delta' j}{2}) = F^B_{j,\mu_1}(\mu_1 j - \frac{\Delta' j}{2}) \le e^{-2j\frac{\Delta'^2}{4}} \le \frac{1}{T^2} <\frac{8}{T^2}.$$
For these values of $s(j)$, we will use the upper bound of $T$. Thus,
\begin{eqnarray*} 
\ExR{\min\{\ExB{X(j,s(j),y) \condB s(j)}, T\}}  & \le & (1-8/T^2) \cdot \left(\frac{1}{(1-8/T^2)} - 1\right)+ \frac{8}{T^2} \cdot T \le \frac{16}{T}.
\end{eqnarray*}

%For large $j$, i.e., $j \ge 4 \log T/ \Delta^2$, we use Chernoff-Hoeffding bounds to argue that with high probability, the number of successes $s$ in $j$ plays of first arm will be close its mean $\mu_1 j$, and greater than $\mu_1 j - \Delta j/2$. And, for $s \ge \mu_1 j - \Delta j/2 = yj + \Delta j /2$, we can show that the probability  $F^B_{j+1, y}(s)$ will be large enough, again using Chernoff-Hoeffding bounds. These observations allows us to derive that $\Ex[Y_{j}] \le 1 + \frac{1}{T}$, for $j \ge L$. 

\paragraph{Case of small $j$:}
For small $j$, the argument is more delicate. We use,
\begin{equation} 
\label{eq:smallj-eq1-tight}
\ExR{\ExB{X(j,s(j), y)\condB s(j)}} = \ExR{\frac{1}{F^B_{j+1,y}(s(j))}-1} = \sum_{s=0}^j \frac{f^B_{j,\mu_1}(s)}{F^B_{j+1,y}(s)} -1,
%= \sum_{s=0}^j \left(\frac{j+1-s}{j+1}\right) \frac{1}{(1-y)} \frac{f^B_{j,\mu_1}(s)}{F^B_{j,y}(s)} 
 \end{equation}
 where $f^B_{j,\mu_1}$ denotes pdf of the $\mbox{Binomial}(j,\mu_1)$ distribution.
We use the observation that for $s\ge \lceil y(j+1) \rceil$, $F^B_{j+1,y}(s) \ge 1/2$. This is because the median of a $\mbox{Binomial}(n,p)$ distribution is either $\lfloor np \rfloor$ or $\lceil np \rceil$ (see \cite{JogdeoS}). Therefore,
\begin{equation} 
\label{eq:larges-bound}
\sum_{s=\lceil y(j+1) \rceil}^j \frac{f^B_{j,\mu_1}(s)}{F^B_{j+1,y}(s)} \le 2.
\end{equation}
For small $s$, i.e., $s \le \lfloor yj \rfloor$, we use $F^B_{j+1, y}(s) = (1-y)F^B_{j,y}(s) + y F_{j,y}(s-1) \ge (1-y)F^B_{j,y}(s)$ and  $F^B_{j,y}(s) \ge f^B_{j,y}(s)$, to get
\begin{eqnarray}
\label{eq:smalls-bound}
\sum_{s=0}^{\lfloor yj \rfloor} \frac{f^B_{j,\mu_1}(s)}{F^B_{j+1,y}(s)} 
%& \le & \sum_{s=0}^{\lfloor yj \rfloor} \frac{1}{(1-y)} \frac{f^B_{j,\mu_1}(s)}{F^B_{j,y}(s)} \nonumber\\
 & \le & \sum_{s=0}^{\lfloor yj \rfloor} \frac{1}{(1-y)} \frac{f^B_{j,\mu_1}(s)}{f^B_{j,y}(s)}\nonumber\\
 & = &  \sum_{s=0}^{\lfloor yj \rfloor} \frac{1}{(1-y)} \frac{\mu_1^s(1-\mu_1)^{j-s}}{y^s (1-y)^{j-s}}\nonumber\\
 & = &  \sum_{s=0}^{\lfloor yj \rfloor} \frac{1}{(1-y)} R^s \frac{(1-\mu_1)^{j}}{(1-y)^{j}}\nonumber\\
 & = &  \frac{1}{(1-y)} \left(\frac{R^{\lfloor yj \rfloor+1}-1}{R-1}\right) \frac{(1-\mu_1)^{j}}{(1-y)^{j}}\nonumber\\
 & \le & \frac{1}{(1-y)}\frac{R}{R-1} \frac{\mu_1^{yj}(1-\mu_1)^{(j-yj)}}{y^{yj}(1-y)^{j-yj}}\nonumber\\
 & = & \frac{\mu_1}{\mu_1-y} e^{-Dj} = \frac{\mu_1}{\Delta'} e^{-Dj}.
 \end{eqnarray}
%where $R = \frac{\mu_1 (1-y)}{(1-\mu_1)y} \ge 1$, $D$ is the KL-divergence $D = y \log \frac{y}{\mu_1} + (1-y) \log \frac{1-y}{1-\mu_1} \ge 0$.

If $\lfloor yj \rfloor < \lceil yj \rceil < \lceil y(j+1) \rceil$, then we need to additionally consider $s=\lceil yj \rceil$. Note, however, that in this case $\lceil yj \rceil \le yj + y$. For $s=\lceil yj \rceil$, 
\begin{eqnarray}
\label{eq:specials-bound1}
\frac{f^B_{j,\mu_1}(s)}{F^B_{j+1,y}(s)} %& = & \frac{f^B_{j,\mu_1}(s)}{(1-y)F^B_{j,y}(s) + y F^B_{j,y}(s-1)} \\
& \le & \frac{1}{(1-y)F^B_{j,y}(s)} \nonumber\\
& \le & \frac{2}{1-y}. 
\end{eqnarray}
Alternatively, we can use the following bound for $s=\lceil yj \rceil$,
\begin{eqnarray}
\label{eq:specials-bound2}
\frac{f^B_{j,\mu_1}(s)}{F^B_{j+1,y}(s)} & \le  & \frac{1}{(1-y)} \frac{f^B_{j,\mu_1}(s)}{F^B_{j,y}(s)} \nonumber\\
& \le & \frac{1}{(1-y)} \frac{f^B_{j,\mu_1}(s)}{f^B_{j,y}(s)} \nonumber\\
& \le & \frac{1}{(1-y)} R^s \left(\frac{1-\mu_1}{1-y}\right)^j \nonumber\\
& \le & \frac{1}{(1-y)} R^{yj+y} \left(\frac{1-\mu_1}{1-y}\right)^j \:\:\:\:\:\;\; (\text{because } s=\lceil yj\rceil \le yj + y) \nonumber\\
& \le & \frac{R^y}{(1-y)} e^{-Dj}.
\end{eqnarray}

Next, we substitute the bounds from \eqref{eq:larges-bound}-\eqref{eq:specials-bound2} in Equation \eqref{eq:smallj-eq1-tight} to get the result in the lemma. In this substitution, for $s=\lceil yj \rceil$, we use the bound in Equation \eqref{eq:specials-bound1} when $j < \frac{y}{D} \ln R$, and the bound in Equation \eqref{eq:specials-bound2} when $j \ge \frac{y}{D} \ln R$.
\end{proof}

\subsection{Details of Equation \eqref{eq:k2T}}
\label{app:eq:k2T}
Using Lemma \ref{lem:exj} for $y=\mu_2+\Delta/2$, and $\Delta'=\Delta/2$, we can bound the expected number of plays of the second arm as:
\begin{eqnarray*}
\ExR{k_2(T)} & = & L  + \ExR{\sum_{j=j_0}^{T-1} Y_{j}}\\
& \le & L + \sum_{j=0}^{T-1} \ExR{\min\{ \ExB{X(j,s(j), \mu_2+\frac{\Delta}{2}) \condB s(j)},\ \ T\}\ }+ \sum_t \Pr(\overline{E_2(t)})\cdot T\\
& \le & L + \frac{4\ln T}{\Delta'^2} + \sum_{j=0}^{4(\ln T)/\Delta'^2-1} \frac{\mu_1}{\Delta'} e^{-Dj} + \left(\frac{y}{D}\ln R\right) \frac{2}{1-y} +  \sum_{j=\frac{y}{D}\ln{R}}^{4(\ln T)/\Delta'^2-1} \frac{R^y e^{-Dj}}{1-y} + \frac{16}{T} \cdot T + 2\\
& = &  L + \frac{4\ln T}{\Delta'^2} + \sum_{j=0}^{4(\ln T)/\Delta'^2-1} \frac{\mu_1}{\Delta'} e^{-Dj} + \frac{y}{D}\ln{R} \cdot \frac{2}{(1-y)} + 
\sum_{j=0}^{4\ln T /\Delta'^2-\frac{y}{D}\ln{R}-1} \frac{1}{1-y} e^{-Dj} + 18\\
& \le &  L + \frac{4\ln T}{\Delta'^2} + \frac{y}{D}\ln{R} \cdot \frac{2}{\Delta'} + \sum_{j=0}^{T-1} \frac{(\mu_1 +1)}{\Delta'} e^{-Dj} + 18\\
 & \stackrel{(*)}{\le} & L + \frac{4\ln T}{\Delta'^2} + \frac{D+1}{\Delta' D} \cdot \frac{2}{\Delta'} +  \frac{2}{\Delta'} \frac{2}{(\min\{D,1\})} + 18\\
& \stackrel{(**)}{\le} & L + \frac{4\ln T}{\Delta'^2} + \frac{2}{\Delta'^2} + \frac{1}{\Delta'^4} + \frac{4}{\Delta'^3} +18\\
& = & L + \frac{16\ln T}{\Delta^2} + \frac{8}{\Delta^2} + \frac{16}{\Delta^4} + \frac{32}{\Delta^3} + 18\\
%& \le &  \frac{18\ln T}{\Delta^2} + \frac{4}{\Delta'^2} + \frac{2}{\Delta'^4} + \frac{2}{\Delta'^3} + 2 \\
& \le &  \remove{\frac{32\ln T}{\Delta^2}} \add{\frac{40\ln T}{\Delta^2}} + \frac{48}{\Delta^4} + 18. 
\end{eqnarray*}
%where the last inequality uses Pinsker's inequality to obtain $D\le 2{\Delta'}^2$.
The step marked $(*)$ is obtained using following derivations.
$$ y\ln R = y\ln \frac{\mu_1(1-y)}{y(1-\mu_1)} = y \ln \frac{\mu_1}{y} + y \ln \frac{(1-y)}{(1-\mu_1)} \le \mu_1 + \frac{y}{1-y} (D - y\ln \frac{y}{\mu_1} )  \le 1+\frac{y}{1-y}(D + \mu_1) \le  \frac{D+1}{\Delta'}.$$
And, since $D\ge 0$ (Gibbs' inequality),
$$ \sum_{j\ge 0} e^{-Dj} = \frac{1}{1-e^{-D}} \le \max\{\frac{2}{D}, \frac{e}{e-1}\} \le \frac{2}{\min\{D,1\}}.$$
And, $(**)$ uses Pinsker's inequality to obtain $D\geq 2{\Delta'}^2$.\newline\\

\subsection{Proof of Lemma~\ref{lem:multiple-E(T)}}
\label{app:lem:multiple-E(T)}
\begin{proof}
The proof of this lemma follows on the similar lines as the proof of Lemma \ref{lem:E2} in Appendix \ref{app:lem:E2} for the two arms case. We will prove the second statement, the first statement will follow as a corollary.

To prove the second statement of this lemma, we are required to lower bound the probability of $\Pr(E(t) | s(j)=s)$ for all $t, j, s\le j$, by $1-\frac{4(N-1)}{T^2}$, where $s(j)$ denotes the number of successes in first $j$ plays of the first arm.
Recall that event $E(t)$ holds if the following is true: 
$$\{\forall i\in C(t), \theta_i(t) \in [\mu_i - \frac{\Delta_i}{2}, \mu_i + \frac{\Delta_i}{2}]\}$$
%$$\{\theta_2(t) \le \mu_2 + \frac{\Delta}{2}\} \mbox{ or } \{k_2(t) < L\}.$$ 
Let us define $E_i^+(t)$ as the event $\{\theta_i(t) \le \mu_i + \frac{\Delta_i}{2} \mbox{ or } i\notin C(t)\}$, and $E_i^-(t)$ as the event $\{\theta_i(t) \ge \mu_i - \frac{\Delta_i}{2} \mbox{ or } i\notin C(t)\}$. Then, we can bound $\Pr(\overline{E(t)} | s(j))$ as
$$\Pr(\overline{E(t)}| s(j)) \le \sum_{i=2}^N \Pr(\overline{E_i^+(t)} | s(j)) + \Pr(\overline{E_i^-(t)} | s(j)).$$
Now, observe that 
$$\Pr(\overline{E_i^+(t)} | s(j)) = 
%\Pr(\theta_i(t) > \mu_i + \frac{\Delta_i}{2}, i\notin C(t) | Z_{1,m}, m=1,\ldots, j) = 
\Pr(\theta_i(t) > \mu_i + \frac{\Delta_i}{2}, k_i(t) \ge L_i | s(j)),$$
where $k_i(t)$ is the number of plays of arm $i$ until time $t-1$.

%We need to define $A(t)$ more precisely. 
As in the case of two arms, define $A_i(t)$ as the event 
$$A_i(t): \frac{S_i(t)}{k_i(t)} \le \mu_i + \frac{\Delta}{4},$$
where $S_i(t), k_i(t)$ denote the number of successes and number of plays respectively of the $i^{th}$ arm until time $t-1$.

We will upper bound the probability of $\Pr(\overline{E_i^+(t)} | s(j))$ for all $t, j, i \ne 1, $ using,
\begin{eqnarray}
\label{eq:Ei}
\Pr(\overline{E_i^+(t)} | s(j) ) & = & \Pr(\theta_i(t) > \mu_i + \frac{\Delta_i}{2}, k_i(t) \ge L_i | s(j)) \nonumber \\
& \le & \Pr(\overline{A_i(t)}, k_i(t) \ge L_i | s(j) ) + \Pr(\theta_i(t) > \mu_i+\frac{\Delta_i}{2}, k_i(t) \ge L_i, A_i(t) | s(j))\nonumber\\
\end{eqnarray}
%where $L=32\ln T/\Delta^2$ (note that earlier, $L$ was $16\ln T/\Delta^2$, this increase in $L$ will case a small increase in the constants in the result.).

For clarity of exposition, similar to the two arms case, for every $i=1,\ldots, N$ we define variables $\{Z_{i,m}\}$, and $\overline{Z}_{i,M}$. $Z_{i,m}$ denote the output of the $m^{th}$ play of the $i^{th}$ arm. And,
$$\overline{Z}_{i,M} = \frac{1}{M} \sum_{m=1}^{M} Z_{i,m}$$
Note that for all $i,m$, $Z_{i,m}$ is Bernoulli variable with mean $\mu_i$, and all $Z_{i,m}, i=1,\ldots, N, m=1, \ldots, T$ are independent of each other.
%Now, we want to prove that for all time $t$, $\Pr(\overline{A(t)}, k_2(t) \ge L)$ is small (this was not stated clearly earlier).
%$\Pr(\frac{Z_{2,k_2(t)}}{k_2(t)} \ge \mu_2 + \frac{\Delta}{4}, k_2(t) = \ell)$ is small.

Now, instead of bounding the first term $\Pr(\overline{A_i(t)}, k_i(t) \ge L_i | s(j) ) $, we prove a bound on $\Pr(\overline{A(t)}, k_2(t) \ge L | Z_{1,1}, \ldots, Z_{1,j})$. Note that the latter bound is stronger, since $s(j)$ is simply $\sum_{m=1}^j Z_{1,m}$.

Now, for all $t$, $i\ne 1$,
$$
\begin{array}{rcl}
\Pr(\overline{A_i(t)}, k_i(t) \ge L_i | Z_{1,1}, \ldots, Z_{1,j}) & =  & \sum_{\ell =L}^{T} \Pr(\overline{Z}_{i,k_i(t)} >  \mu_i + \frac{\Delta_i}{4}, k_i(t) = \ell | Z_{1,1}, \ldots, Z_{1,j}) \\
& = & \sum_{\ell=L}^T\Pr(\overline{Z}_{i,\ell} > \mu_i + \frac{\Delta_i}{4}, k_i(t) = \ell | Z_{1,1}, \ldots, Z_{1,j}) \\
& \le & \sum_{\ell=L}^T\Pr(\overline{Z}_{i,\ell} > \mu_i + \frac{\Delta_i}{4} | Z_{1,1}, \ldots, Z_{1,j}) \\
& = & \sum_{\ell=L}^T\Pr(\overline{Z}_{i,\ell} > \mu_i + \frac{\Delta_i}{4}) \\
& \le & \sum_{\ell=L}^T e^{-2\ell\Delta_i^2/16} \\
& \le & \frac{1}{T^2}
\end{array}
$$
The third last equality holds because for all $i, i', m, m'$, $Z_{i,m}$ and  $Z_{i',m'}$ are independent of each other, which means $\overline Z_{i,\ell}$ is independent of $Z_{1,m}$ for all $m=1,\ldots, j$.
The second last inequality is by applying Chernoff bounds, since $\overline{Z}_{i,\ell}$ is simply the average of $\ell$ iid Bernoulli variables each with mean $\mu_2$. \newline\\

We will derive the bound on second probability term in \eqref{eq:Ei} in a similar manner. As before, it will be useful to define $W(\ell, z)$ as a random variable distributed as $\mbox{Beta}(\ell z+1, \ell-\ell z+1)$. Note that if at time $t$, the number of plays of arm $i$ is $k_i(t)=\ell$, then $\theta_i(t)$ is distributed as $\mbox{Beta}(\ell \overline{Z}_{i,\ell}+1, \ell-\ell \overline{Z}_{i,\ell}+1)$, i.e. same as $W(\ell, \overline{Z}_{i,\ell})$. Now, for the second probability term in \eqref{eq:Ei}, 
\begin{eqnarray*}
& & \hspace{-0.5in}\Pr(\theta_i(t) > \mu_i + \frac{\Delta}{2}, A_i(t), k_i(t) \ge L_i | Z_{1,1}, \ldots, Z_{1,j} )\\
& = & \sum_{\ell=L_i}^T \Pr(\theta_i(t) > \mu_i + \frac{\Delta_i}{2}, A_i(t), k_i(t) = \ell | Z_{1,1}, \ldots, Z_{1,j})\\
& \le & \sum_{\ell=L_i}^T \Pr(\theta_i(t) > \frac{S_i(t)}{k_i(t)} - \frac{\Delta_i}{4} + \frac{\Delta_i}{2}, k_i(t) =\ell | Z_{1,1}, \ldots, Z_{1,j})  \\
& = & \sum_{\ell=L_i}^T \Pr(W(\ell, \overline{Z}_{i,\ell}) > \overline{Z}_{i,\ell} + \frac{\Delta_i}{4}, k_i(t) =\ell | Z_{1,1}, \ldots, Z_{1,j}) \\
& \le & \sum_{\ell=L_i}^T \Pr(W(\ell, \overline{Z}_{i,\ell}) > \overline{Z}_{i,\ell} + \frac{\Delta_i}{4} | Z_{1,1}, \ldots, Z_{1,j}) \\
& = & \sum_{\ell=L_i}^T \Pr(W(\ell, \overline{Z}_{i,\ell}) > \overline{Z}_{i,\ell} + \frac{\Delta_i}{4}) \\
\mbox{(using Fact 1) } & = & \sum_{\ell=L_i}^T \ExR{F^B_{\ell+1, \overline{Z}_{i,\ell} +\frac{\Delta_i}{4}}(\ell\overline{Z}_{i,\ell})}  \\
& \le & \sum_{\ell=L_i}^T \ExR{F^B_{\ell, \overline{Z}_{i,\ell}+\frac{\Delta_i}{4}}(\ell\overline{Z}_{i,\ell})}\\
& \le & \sum_{\ell=L_i}^T \exp\{-\frac{2\Delta_i^2\ell^2/16}{\ell}\} \\
& \le & T e^{-2L_i\Delta_i^2/16} = \frac{1}{T^2}.
\end{eqnarray*}
%Here, we used the observation that when $k_i(t)=\ell$, $\theta_i(t)$ is distributed as $\mbox{Beta}(\ell \overline{Z}_{i,\ell}+1, \ell-\ell \overline{Z}_{i,\ell}+1)$, and then used Fact \ref{fact:beta-binomial} for relating cdf of Beta distribution to cdf of Binomial distribution.
Here, we used the observation that for all $i, i', m, m'$, $Z_{i,m}$ and $Z_{i',m'}$ are independent of each other, which means $\overline Z_{i,\ell}$ and $W(\ell, \overline{Z}_{i,\ell})$ are independent of $Z_{1,m}$ for all $m=1,\ldots, j$.
The third-last inequality follows from the observation that 
$$F^B_{n+1, p}(r) = (1-p)F^B_{n,p}(r) + p F^B_{n,p}(r-1) \le (1-p)F^B_{n,p}(r) + p F^B_{n,p}(r) = F^B_{n,p}(r). $$
And, the second-last inequality follows from Chernoff--Hoeffding bounds 
(refer to Fact \ref{lem:Chernoff} and Lemma \ref{lem:Binomial}).
Substituting above in Equation \eqref{eq:Ei}, we get
$$\Pr(\overline{E_i^+(t)} | s(j)) \le \frac{2}{T^2}$$
Similarly, we can obtain
 $$\Pr(\overline{E_i^-(t)} | s(j)) \le \frac{2}{T^2}$$
 
Summing over $i=2,\ldots, N$, we get
$$\Pr(\overline{E(t)} | s(j)) \le \frac{4(N-1)}{T^2}$$
which implies the second statement of the lemma. The first statement is a simple corollary of this.

\end{proof}

\subsection{Proof of Lemma \ref{lem:GammaV-dep}}
\label{app:GammaV-dep}
\begin{proof}
$$
\mbox{$\ExR{\sum_{\ell=1}^{\gamma_j+1} V_j^{\ell,a} \condR s(j)} =  \ExR{\sum_{\ell=1}^{T}   V_j^{\ell,a}\cdot \ \mbox{I}(\gamma_j \ge \ell-1) \condR s(j)}$}
$$
Let ${\cal F}_{\ell-1}$ denote the history until before the beginning of interval $I_j(\ell)$ (i.e. the values of $\theta_i(t)$ and the outcomes
of playing the arms until the time step before the first time step of $I_j(\ell)$). Note that the value of random variable $\mbox{I}(\gamma_j \ge \ell-1)$ is completely determined by  ${\cal F}_{\ell-1}$.
Therefore,
\begin{eqnarray*}
& & \mbox{$\ExR{\sum_{\ell=1}^{\gamma_j+1}  V_j^{\ell,a} \condR s(j)}$} \\
%& = & \mbox{$\ExR{\sum_{\ell=1}^{T} \left(\sum_a \Delta_a V_j^{\ell,a}\right)	\cdot \ \mbox{I}(\gamma_j \ge \ell-1)\} \condR s(j)}$} \\
& = & \mbox{$\ExR{\sum_{\ell=1}^{T} \ExB{  V_j^{\ell,a}	\cdot \ \mbox{I}(\gamma_j \ge \ell-1)\ \ \condB s(j), {\cal F}_{\ell-1}} \condR s(j) }$} \\
& = & \mbox{$\ExR{\sum_{\ell=1}^{T} \ExB{  V_j^{\ell,a}\ \ \condB s(j), {\cal F}_{\ell-1}} \cdot \ \mbox{I}(\gamma_j \ge \ell-1) \condR s(j)}$}. \\
\end{eqnarray*}
%Now, the length of interval $I_j(\ell)$ is determined by the values taken by random variables $\theta_{1}(t)$ and $\theta_i(t), {i\ne 1}$ for time steps $t$ starting at $t_{\ell}$. Because the distribution of $\theta_1(t)$ does not change during the interval $I_j(\ell)$, $\theta_1(t)$ for $t\ge t_{\ell}$ are iid and independent of ${\cal F}_{\ell-1}$. However, $\theta_i(t), {i\ne 1}$ for $t\ge t_{\ell}$ could be correlated to $\theta_i(t)$ for $t < t_{\ell}$. 

%\begin{eqnarray*}
%\ExB{|I_j(\ell)|	\ \ \condB s(j), E(T), {\cal F}_{\ell-1}} & = & \ExB{\sum_{a =2}^N V_j^{\ell, a} 	\ \ \condB s(j), E(T), {\cal F}_{\ell-1}} 
%\end{eqnarray*}
Recall that $V^{\ell,a}_{j}$ is the number of contiguous steps $t$ for which $a$ is the best arm in saturated set $C(t)$ and iid variables $\theta_1(t)$ have value smaller than \remove{$\max_{i\in C(t)} \theta_i(t)$} \add{$\mu_{a}+\frac{\Delta_{a}}{2}$}. 
\remove{Now, given event $E(T)$, for all $t$ we have $\max_{i\in C(t)} \theta_i(t)\le \mu_{a_t}+\frac{\Delta_{a_t}}{2}$, where $a_t$ is the best arm in $C(t)$.}
Observe that given $s(j)=s$ \remove{$E(T)$,} and ${\cal F}_{\ell-1}$, $V^{\ell,a}_{j}$ is the length of an interval which ends when the value of an iid $\Beta(s+1,j-s+1)$ distributed variable exceeds \remove{a quantity smaller than} $\mu_{a}+\frac{\Delta_{a}}{2}$ (i.e., $M(t)$ happens), or if an arm other than $a$ becomes the best saturated arm, or if we reach time $T$. Therefore, given \remove{$E(T),$} $s(j), {\cal F}_{\ell-1}$, $V_j^{\ell, a}$ is stochastically dominated by $\min\{ X(j,s(j), \mu_a+\frac{\Delta_a}{2}), T\}$, where recall that $X(j,s(j), y)$ was defined as the number of trials until an independent sample from $\mbox{Beta}(s+1,j-s+1)$ distribution exceeds $y$.
That is, for all $a$,
\begin{eqnarray*} 
\label{eq:geometric-domination}
\mbox{$\ExB{V^{\ell,a}_j \condB s(j)\remove{, E(T)}, {\cal F}_{\ell-1}}$} & \le & \mbox{$\ExB{\min\{X(j,s(j), \mu_a + \frac{\Delta_a}{2}), T\} \condB s(j),\remove{ E(T),} {\cal F}_{\ell-1}}$} \\
& = & \mbox{$\ExB{\min\{X(j,s(j), \mu_a + \frac{\Delta_a}{2}), T\} \condB s(j)}$}.
\end{eqnarray*}
%This gives,
%\begin{eqnarray*}
%\ExB{|I_j(\ell)|	\ \ \condB s(j), E(T), {\cal F}_{\ell-1}} & \le & \sum_{a=2}^N \ExB{X(j,s(j), \mu_a + \frac{\Delta_a}{2}) \condB s(j)}.
%\end{eqnarray*}
%Or, for regret,
%\begin{eqnarray*}
%\ExB{\sum_{a=2}^N V_j^{\ell,a} \Delta_a	\ \ \condB s(j), E(T), {\cal F}_{\ell-1}} & \le & \sum_{a=2}^N \Delta_a \ExB{X(j,s(j), \mu_a + \frac{\Delta_a}{2}) \condB s(j)}.
%\end{eqnarray*} 
Substituting, we get,
\begin{eqnarray*}
& & \mbox{$\hspace{-0.3in}\ExR{\sum_{\ell=1}^{\gamma_j+1} V_j^{\ell,a} \ \condR	s(j)\remove{, E(T)}}$} \\
& \le & \mbox{$\ExR{\sum_{\ell=1}^{T}  \ExB{\min\{X(j,s(j), \mu_a + \frac{\Delta_a}{2}), T\} \condB s(j)} \cdot \ \mbox{I}(\gamma_j \ge \ell-1)\ \ \condR s(j)\remove{, E(T)} } $}\\
& = & \mbox{$  \ExB{\min\{X(j,s(j), \mu_a + \frac{\Delta_a}{2}), T\} \condR s(j)} \cdot \ExR{\sum_{\ell=1}^T \mbox{I}(\gamma_j \ge \ell-1) \condB s(j)} $}\\
& = & \mbox{$ \ExB{\min\{X(j,s(j), \mu_a + \frac{\Delta_a}{2}), T\} \condR s(j)} \cdot \ExR{\gamma_j+1 \condB s(j)} $}.
\end{eqnarray*}
\add{
This immediately implies,
$$\mbox{$\ExR{\sum_{a=2}^N \Delta_a \ExR{\sum_{\ell=1}^{\gamma_j+1} V_j^{\ell,a} \ \condR	s(j)}} \le \ExR{\sum_{a=2}^N \Delta_a \ExB{\min\{X(j,s(j), \mu_a + \frac{\Delta_a}{2}), T\} \condR s(j)} \cdot \ExR{\gamma_j+1 \condB s(j)} } $}$$
}
%Now note that the quantities inside the outer sum do not depend on $\ell$, and thus we get the expression claimed in the lemma.
\end{proof}
\section{Proof of Theorem \ref{th:N-arms}: details}
\label{app:N-arms-details}
We continue the proof from the main body of the paper. 
For the first term in Equation \eqref{eq:tmp3}, 
\begin{eqnarray}
\label{eq:tmp4}
\ExR{\frac{1}{F_{j^*_a+1, y_a}(s(j^*_a))}\cdot I(s(j^*_a) \le \lfloor y_aj^*_a\rfloor)} & \le & \sum_j \ExR{\frac{1}{F_{j+1, y_a}(s(j))}\cdot I(s(j) \le \lfloor y_aj\rfloor)} \nonumber\\
& = & \hspace{-0.02in}\sum_j \sum_{s=0}^{\lfloor y_aj\rfloor} \frac{f_{j,\mu_1}(s)}{F_{j+1, y_a}(s)} \le \sum_j \frac{\mu_1}{\Delta'_a} e^{-D_a j} \le \frac{16}{\Delta_a^3},
\end{eqnarray}
where $\Delta'_a = \mu_1-y_a = \Delta_a/2$, $D_a$ is the KL-divergence between Bernoulli distributions with parameters $\mu_1$ and $y_a$. The penultimate inequality follows using \eqref{eq:smalls-bound} in the proof of Lemma \ref{lem:exj} in \add{Appendix \ref{app:lem:exj}}, with $\Delta'=\Delta'_a$, and $D = D_a$. The last inequality uses the geometric series sum (note that $D_a\ge 0$ by Gibbs' inequality).
\begin{center}
$\sum_{j} e^{-D_a j} \le \frac{1}{1-e^{-D_a}} \le \max\{ \frac{2}{D_a}, \frac{e}{e-1}\} \le \frac{2}{\min\{D_a, 1\}} \le \frac{2}{{\Delta'_a}^2} = \frac{8}{\Delta_a^2}.$
\end{center}
And, for the second term, using the fact that $F_{j+1,y}(s) \ge (1-y)F_{j,y}(s)$, and that for $s\ge \lceil yj \rceil $, $F_{j,y}(s) \ge 1/2$ (Fact \ref{fact:median}),
\begin{equation}
\label{eq:tmp5}
\ExR{\frac{1}{F_{j^*_a+1, y_a}(s(j^*_a))}\cdot I(s(j^*_a) \ge \lceil y_aj^*_a\rceil)}\le \frac{2}{1-y_a} \le \frac{4}{\Delta_a}.
\end{equation}
Substituting the bound from Equation \eqref{eq:tmp4} and \eqref{eq:tmp5} in Equation~\eqref{eq:tmp3},
\begin{equation}
\label{eq:tmp6}
\mbox{$\sum_{j=0}^{T-1} \ExR{\ExB{\gamma_j|s(j)} \sum_a 3\Delta_a \ExB{X(j,s(j), y_a)|s(j)} } %\le (\sum_u L_u )\sum_a (\frac{48}{\Delta_a^2} + \frac{12\Delta_a}{\Delta_a}) 
\le (\sum_u L_u + 4(N-1))\sum_a (\frac{48}{\Delta_a^2} + 12). $}
\end{equation}

Also, using Lemma \ref{lem:exj} while substituting $y$ with $y_a=\mu_a+\frac{\Delta_a}{2}$ and $\Delta'$ with  $\mu_1-y_a=\frac{\Delta_a}{2}$, 
\begin{eqnarray}
\label{eq:tmp7}
& & \hspace{-1in}\sum_{j=0}^{T-1} \sum_{a=2}^N (3\Delta_a) \ExR{\ExB{\min\{X(j,s(j), \mu_a + \frac{\Delta_a}{2}), T\} \condB s(j)}} \nonumber\\
& \le & \sum_a (3\Delta_a) \sum_{j=0}^{\frac{16(\ln T)}{{\Delta_a}^2}-1} \left(1+ \frac{2}{1-y_a}\right) + \sum_{j\ge\frac{16(\ln T)}{{\Delta_a}^2}}^T(3\Delta_a) \frac{16}{T} \nonumber\\
& \le & \sum_a \frac{48\ln T}{\Delta_a} + \frac{192}{\Delta_a^2} + 48\Delta_a. 
%& \le & \sum_a (3\Delta_a \ln 
%\sum_{j=0}^{4\ln T/\Delta'^2_a-1} \sum_a 2\Delta_a (1+\frac{2}{1-y_a}  + \frac{\mu_1}{\Delta'_a}e^{-D_aj}) + \sum_a 2\Delta_a \nonumber \\
%& \le & \sum_a 2\Delta_a (\frac{4 \ln T}{{\Delta'}^2_a} + \frac{8\ln T}{\Delta'^3_a} + \frac{2}{{\Delta'_a}^3}) + \sum_a 2\Delta_a \nonumber\\
%& \le & \sum_a 2\Delta_a (\frac{16 \ln T}{{\Delta}^2_a} + \frac{64\ln T}{\Delta^3_a} + \frac{16}{{\Delta_a}^3}) + \sum_a 2\Delta_a \nonumber\\
%& \le & \sum_a \frac{32 \ln T}{{\Delta}_a} + \frac{128\ln T}{\Delta^2_a} + \frac{32}{{\Delta_a}^2} + \sum_a 2\Delta_a \nonumber
%%%%& \le & 2\Delta_a(\frac{16\ln T}{\Delta^2_a} + \frac{64\ln T}{\Delta_a^3} + \frac{8}{\Delta^3_a}) + \sum_a 2\Delta_a
\end{eqnarray}

Substituting bounds from \eqref{eq:tmp6} and \eqref{eq:tmp7} in Equation \eqref{eq:GammaV-dep}, 

\begin{eqnarray*}
& & \hspace{-0.2in}\sum_{j=0}^{T-1} \ExR{\sum_{\ell=1}^{\gamma_j+1} \sum_a V_j^{\ell,a} 3\Delta_a \remove{|	E(T)}}  \\
& \le & (\sum_u L_u + 4(N-1))\sum_a (\frac{48}{\Delta_a^2} + 12) + \sum_a (\frac{48\ln T}{\Delta_a} + \frac{192}{\Delta_a^2} + 48\Delta_a) \\
%& = & (\sum_u \frac{16\ln T}{\Delta_u^2}) (\sum_a \frac{32}{\Delta_a^2} + 8) + \sum_a \frac{32 \ln T}{\Delta_a} + \frac{128}{\Delta_a^2} +32\Delta_a\vspace{0.1in}\\
& \le & \remove{768} \add{1152} (\ln T)(\sum_i \frac{1}{\Delta_i^2})^2 + \remove{192} \add{288} (\ln T) \sum_i \frac{1}{\Delta_i^2} + 48 (\ln T) \sum_a \frac{1}{\Delta_a} + 192 \add{N} \sum_a \frac{1}{\Delta_a^2} + \add{96}(N-1).
\end{eqnarray*}
Now, using the result that $\remove{\Pr(\overline{E(T)}) \le 4(N-1)/T} \add{\Pr(\overline{E(t)}) \le 4(N-1)/T^2}$ (\text{by Lemma~\ref{lem:multiple-E(T)}}) with Equation \eqref{eq:R-V},  we can bound the total regret due to playing saturated arms as
\begin{eqnarray*}
\Ex[{\cal R}^s(T)] & =  & \sum_j \Ex[{\cal R}^s(I_j) \remove{|E(T)}] \remove{ + T \cdot \Pr(\overline{E(T)})} \\
& = & \sum_j \ExR{\sum_{\ell=1}^{\gamma_j+1} \sum_a V_j^{\ell,a} \add{3}\Delta_a \remove{|	E(T)}} +\add{2T \cdot \sum_t \Pr(\overline{E(t)})}  \remove{4(N-1)}\\
       & \leq    & \add{1152} \remove{768} (\ln T)(\sum_i \frac{1}{\Delta_i^2})^2 + \add{288}\remove{192} (\ln T) \sum_i \frac{1}{\Delta_i^2} \\
       & & \hspace{1in} + 48 (\ln T) \sum_a \frac{1}{\Delta_a} + 192 \add{N} \sum_a \frac{1}{\Delta_a^2} + \add{96}(N-1) + 8(N-1).
       %& \leq   512 (\ln T) \eta^2+ 160 (\ln T) \eta + 128 \eta + 22 (N-1) & 
\end{eqnarray*}
%where $\eta=\sum_{i=2}^N \frac{1}{\Delta_i^2}$.
Since an unsaturated arm $u$ becomes saturated after $L_u$ plays, regret due to unsaturated arms is at most 
$$\Ex[{\cal R}^u(T)] \le \sum_{u=2}^N L_u \Delta_u  = \remove{16} \add{24} (\ln T) \left(\sum_{u=2}^N \frac{1}{\Delta_u}\right). $$

Summing the regret due to saturated and unsaturated arms, we obtain the result of Theorem \ref{th:N-arms}.\newline\\

The proof for the alternate bound in Remark \ref{rem:N-arms} will essentially follow the same lines except that instead of dividing the interval $I_j(\ell)$ into subdivisions $V_j^{\ell,a}$, we will simply bound the regret due to saturated arms by number of plays times $\Delta_{max}$. That is, we will use the bound,
$$\Ex[{\cal  R}(I_j)]  \le \Ex[ \sum_{\ell=1}^{\gamma_j+1} |I_j(\ell)| \cdot \Delta_{max}]$$
To bound $\Ex[\sum_{\ell=1}^{\gamma_j+1} |I_j(\ell)|]$, we follow the proof for bounding $\Ex[\sum_{\ell=1}^{\gamma_j+1} V^{\ell,\bar{a}}_j]$ for $\bar{a}=\arg \max_{i\ne 1} \mu_i$, i.e.,  replacing $\mu_a$ with $\mu_{\bar{a}} = \max_{i\ne 1} \mu_1$, and $\Delta_a$ with $\Delta_{min}$. In a manner similar to Lemma \ref{lem:GammaV-dep}, we can obtain 
$$\Ex[\sum_{\ell=1}^{\gamma_j+1} |I_j(\ell)| \remove{| E(T)}] \le \Ex[(\gamma_j+1) \min\{X(j,s(j), \mu_M+\frac{\Delta_{min}}{2}), T\} \remove{| E(T)}] +  \add{\Ex[\sum_{t\in I_j} T\cdot I(\overline{E(t)})]}$$
And, consequently, using Equation \eqref{eq:tmp3}, and Equation \eqref{eq:tmp4}--\eqref{eq:tmp7}, and Lemma \ref{lem:multiple-E(T)}, we can obtain
$$\sum_j \Ex[\sum_{\ell=1}^{\gamma_j+1} |I_j(\ell)|] \le O((\sum_u L_u) \frac{1}{\Delta_{min}^3}) = O(\frac{1}{\Delta_{min}^3}\left(\sum_{a=2}^N \frac{1}{\Delta_a^2}\right) \ln T),$$
giving a regret bound of $O(\frac{\Delta_{max}}{\Delta_{min}^3}\left(\sum_{a=2}^N \frac{1}{\Delta_a^2}\right) \ln T)$.

\end{document}